\documentclass{article}

\usepackage[preprint]{neurips_2026}

\usepackage{microtype}

\usepackage{algorithm, algorithmicx, algpseudocode}
\usepackage{amssymb}
\usepackage{amsmath}
\usepackage{graphicx}
\usepackage{mathtools}
\usepackage{needspace} 
\usepackage{multicol}
\usepackage{listings}
\usepackage{amsthm}
\usepackage[shortlabels]{enumitem}
\usepackage[italicdiff]{physics}
\usepackage{cancel}
\usepackage[dvipsnames,table]{xcolor}
\usepackage{verbatim}
\usepackage{comment}
\usepackage{array}
\usepackage{multirow}
\usepackage{tikz-cd}
\usepackage{import}
\usepackage{booktabs}
\usepackage{tabularx}
\usepackage[colorlinks, linktoc=page, linkcolor=MidnightBlue, citecolor=blue]{hyperref}
\usepackage{forest}
\usepackage{float}
\usepackage{xspace}
\usepackage{aliascnt}
\usepackage{wrapfig}
\usepackage{thm-restate}
\usepackage{adjustbox}
\usepackage{bm}
\usepackage{bbm}
\usepackage{complexity}
\usepackage{makecell}
\usepackage{nicefrac}
\usepackage{mathrsfs}
\usepackage[most]{tcolorbox}
\usepackage[capitalize]{cleveref}
\usetikzlibrary{calc, positioning, external, arrows.meta}

\usepackage{booktabs}      
\usepackage{tabularray}    
\usepackage{pifont}        

\usepackage{pgfplots}
\usepgfplotslibrary{groupplots}
\pgfplotsset{compat=1.18}

\allowdisplaybreaks

\let\E\relax

\DeclareMathOperator*{\argmax}{argmax}
\DeclareMathOperator*{\E}{\mathbb E}

\let\cite\citep

\newfloat{protocol}{tbp}{lop}
\floatname{protocol}{Protocol}

\newtheorem{theorem}{Theorem}
\numberwithin{theorem}{section}
\newtheorem*{theorem*}{Theorem}

\newtheorem{lemma}[theorem]{Lemma}

\newtheorem*{proposition*}{Proposition}
\newtheorem{corollary}[theorem]{Corollary}

\theoremstyle{definition}
\newtheorem{definition}[theorem]{Definition}

\newtheorem{ass}[theorem]{Assumption}

\newcommand{\ind}[1]{\mathbbm{I}\qty{#1}}

\newcommand{\Prob}{\mathbb P}
\newcommand{\diff}{\text{d}}

\newcommand{\rX}{x}
\newcommand{\rY}{y}
\newcommand{\fwithin}{f_\star}
\newcommand{\inpdist}{\Ds}
\newcommand{\halt}{\mathbbm{T}}

\newcommand{\As}{\mathcal A}
\newcommand{\Bs}{\mathcal B}

\newcommand{\Ds}{\mathcal D}

\newcommand{\Fs}{\mathcal F}

\newcommand{\Hs}{\mathcal H}

\newcommand{\Os}{\mathcal O}

\newcommand{\Xs}{\mathcal X}

\definecolor{C0}{RGB}{0, 114, 178}
\definecolor{C1}{RGB}{213, 94, 0}
\definecolor{C2}{RGB}{0, 158, 115}
\definecolor{C3}{RGB}{128, 128, 128}

\newcommand{\defpoly}{\bm{\textcolor{C0}{\mathsf{D2}}}}
\newcommand{\defeff}{\bm{\textcolor{C1}{\mathsf{D1}}}}
\newcommand{\defpolysub}{\bm{\textcolor{C0}{\mathsf{D2+}}}}
\newcommand{\defbounded}{\bm{\textcolor{C3}{\mathsf{D3}}}}

\usepackage{todonotes}

\title{Sharper Guarantees for Misspecified Kernelized Bandit Optimization}

\author{%
  Davide Maran \\
  Politecnico di Milano\\
  \texttt{davide.maran@polimi.it} \\
  \And
  Csaba Szepesv\'ari\\
University of Alberta\\
  \texttt{szepesva@ualberta.ca} \\
}

\begin{document}

\maketitle

\begin{abstract}
Existing guarantees for misspecified kernelized bandit optimization pay for misspecification through kernel complexity: 
in generic offline bounds, the misspecification level $\varepsilon$ is multiplied by $\sqrt{d_\mathrm{eff}}$, where $d_\mathrm{eff}$ is the kernel effective dimension, while in online regret bounds, the corresponding penalty is $\sqrt{\gamma_n}\,n\varepsilon$, where $\gamma_n$ is the maximum information gain after $n$ rounds of interaction.
 In this work, we show that, for a large class of kernels, the misspecification amplification can be reduced to logarithmic or polylogarithmic growth. 
In the offline setting, we first prove high-probability simple-regret bounds whose misspecification term is governed by a spectral Lebesgue constant. 
This yields logarithmic amplification for one-dimensional monotone spectra and polylogarithmic amplification for multivariate Fourier-diagonal product kernels. 
In the online setting, we modify a domain-splitting algorithm and prove a cumulative regret bound of $\widetilde{\mathcal O}(\sqrt{\gamma_n n}+n\varepsilon)$ under mild localized eigendecay assumptions, removing the extra $\sqrt{\gamma_n}$ factor from the misspecification term. 
The common principle is localization: spectral localization controls the Lebesgue constant of the offline approximation operator, while domain splitting implements the spatial analogue of this mechanism in the online setting, preventing local misspecification errors from being amplified globally.
\end{abstract}

\section{Introduction}

In stochastic bandits and reinforcement learning, the learner must control errors pointwise to identify high-value actions; in contrast, supervised learning often succeeds with average-error guarantees. This difference makes misspecification substantially more harmful in these settings \cite{lattimore2020bandit}. The issue is not merely statistical noise but geometric amplification: even in linear approximation with adaptive data collection, pointwise error can scale as $\Theta(\sqrt d\,\varepsilon)$ \cite{du2020good,lattimore2020learning}. This may lead to a paradoxical situation where increasing model complexity actually hurts performance (even for infinite data, regardless of overfitting).

This phenomenon is not merely conceptual: it is formalized by lower- and upper-bound analyses in misspecified linear settings \cite{du2020good,lattimore2020learning}. More recently, \cite{maran2025beyond} identified the Lebesgue constant of the random-design projection operator as the exact geometric quantity governing the amplification of misspecification in least-squares regression, and showed that this dependence is intrinsic to squared-loss minimization. The natural next question is whether a comparable operator-theoretic picture can be developed for kernels, whose model class is infinite-dimensional and whose spectral structure matters. In the worst case, kernelized guarantees are still pessimistic \cite{bogunovic2021misspecified}. Our main message is that this worst-case behavior is often avoidable under additional spectral structure: in the offline setting, square-root amplification can be improved to logarithmic or polylogarithmic growth, depending on the kernel class.

Specifically, we investigate both offline optimization, where the agent operates on a fixed dataset of input queries, and the online setting, where the agent must adaptively interact with the environment to minimize regret. In the offline case, we prove logarithmic misspecification amplification in the one-dimensional monotone-spectrum setting and extend the analysis to multivariate product kernels, where the amplification remains polylogarithmic, for example of order $\log^{2m-1}(e+\kappa^m/\tau)$. In the online case, we demonstrate that under a localized eigendecay assumption—restricted to specific subdomains—the misspecification term in the regret can be reduced to $n\varepsilon$ up to logarithmic factors, removing the extra $\sqrt{\gamma_n}$ factor present in prior work. Both results instantiate the same high-level principle: misspecification becomes less harmful when the approximation problem can be localized. The two settings realize this principle differently:
\begin{itemize}
    \item In the offline analysis, localization takes a spectral form. This is the kernel counterpart of the viewpoint of \cite{maran2025beyond}: the misspecification term is governed by the Lebesgue constant of the approximation operator. The new question in the kernel setting is when spectral structure forces this constant to be much smaller than the generic $\sqrt{d_\text{eff}}$ bound.
    \item In the bandit setting, localization takes a spatial form. We show that the domain-splitting algorithm $\pi$-\textsc{GP-UCB} proposed by \cite{janz2020bandit} is sufficient for this purpose, albeit requiring a different bonus construction, refined hyperparameter selection, and a partially novel analysis. Conceptually, this is the online analogue of spatial-localization mechanisms that reduce global misspecification amplification: the algorithm fits local KRR estimators and controls information gain region by region.
\end{itemize}

For readability, we summarize the main quantitative comparison with the most relevant baseline guarantees in \cref{tab:results_comparison}. The table is meant as a comparison of leading misspecification terms; stochastic and complexity terms are kept implicit. The detailed derivations are developed in \cref{sec:offline,sec:online}.

\begin{table}[t]
\centering
\footnotesize
\caption{High-level comparison of leading misspecification-dependent terms under the assumptions listed in each row.}
\label{tab:results_comparison}
\setlength{\tabcolsep}{3pt}
\begin{tabularx}{\linewidth}{>{\raggedright\arraybackslash}p{0.10\linewidth}>{\raggedright\arraybackslash}p{0.21\linewidth}>{\raggedright\arraybackslash}X>{\raggedright\arraybackslash}X}
\toprule
Setting & Assumptions & Baseline / prior guarantee & This work \\
\midrule
Offline & i.i.d. spectral-basis setting; misspecification in $\|f-\fwithin\|_\infty$ & Generic bound $\Lambda(\mathsf{P}_{\tau})\le \sqrt{d_\mathrm{eff}(k|\tau)}$ (Thm.~\ref{thm:boundlebe}, proved here), yielding $\sqrt{d_\mathrm{eff}}\,\varepsilon$ & 1D monotone spectrum: $\log(e+\kappa/\tau)\,\varepsilon$; Fourier-diagonal product kernels on $[-1,1]^m$: $\log^{2m-1}(e+\kappa^m/\tau)\,\varepsilon$ \\
\addlinespace[2pt]
Online & Adaptive kernelized bandits with subdomain eigendecay and bounded eigenfunctions & \cite{bogunovic2021misspecified}: $\widetilde{\mathcal O}(\gamma_n\sqrt n + \sqrt{\gamma_n}\,n\varepsilon)$ & $\widetilde{\mathcal O}(\sqrt{\gamma_n n} + n\varepsilon)$; removes the extra $\sqrt{\gamma_n}$ factor from the misspecification term \\
\bottomrule
\end{tabularx}
\end{table}

\section{Problem Definition}

Let $\mathbb R$ denote the set of reals, $\Xs\subseteq [-1,1]^m$ a measurable input space, and $f: \Xs \to \mathbb R$ an unknown measurable \emph{target} (or regression) function that we wish to optimize by means of a kernel function $k:\Xs^2\to \mathbb R$. In this paper, $k$ will always satisfy the following.
\begin{ass}\label{ass:valboundisoker}
    The kernel $k:\Xs^2\to \mathbb R$ is positive semidefinite, stationary (i.e. there exists $\mathscr K:\Xs-\Xs\to \mathbb R$ such that $k(x,y)=\mathscr K(x-y)$ for all $x,y\in \Xs$), and satisfies $\sup_{u\in \Xs-\Xs}|\mathscr K(u)|\le \kappa$ for some $\kappa\ge 1$.
\end{ass}
The agent interacts with $f$ by querying points $\bm x=(x_t)_{t=1}^n\subset \Xs$ and receiving labels $\bm y=(y_t)_{t=1}^n\subset \mathbb R$ such that
\[f(\rX_t) \coloneqq \E[\rY_t|\rX_t]\]
almost surely for each $t\in [n]:=\{1,\dots,n\}$.
Regarding the choice of the query points, we study two different scenarios.
\begin{enumerate}
    \item (Offline optimization) data $(\bm x,\bm y) = ((\rX_1, \rY_1),\dots,(\rX_n,\rY_n))$ are collected as independent pairs, with the inputs sampled from a common probability distribution $\inpdist$ supported on $\Xs$ and the labels being $R$-subgaussian conditionally on the input.
    \item (Online optimization) the agent sequentially queries points $\rX_t$, receiving a label $y_t$ that is conditionally independent of the past and $R$-subgaussian given the input.
\end{enumerate}
Most results in this paper build on approximating $f$ via kernel ridge regression (KRR).
For $\tau>0$, the KRR estimator on $(\bm x, \bm y)$ is given by
\begin{equation}
    \mathsf{P}_{\bm x,\tau}\bm y(x):=\underbrace{[k(x_1,x),\dots,k(x_n,x)]}_{\Psi_{\bm x}(x)^\top}(K+n\tau I)^{-1}[y_1,\dots,y_n]^\top
    \label{eq:krr},
\end{equation}
where $K$ is the kernel matrix such that $K_{ij}=k(x_i,x_j)$.
In the offline part of the paper, we adopt the normalized regularization convention $K+n\tau I$ rather than $K+\lambda I$.
This estimator is closely linked to the posterior variance of the process, which is given by
$$\sigma_{\bm x,\tau}(x) \coloneqq \sqrt{k(x,x)-\Psi_{\bm x}(x)^\top (K+n\tau I)^{-1}\Psi_{\bm x}(x)}.$$ 
The two parameterizations are equivalent: the regularization $\lambda$ in the $K+\lambda I$ convention corresponds to $\tau=\lambda/n$ here, so in particular $\lambda=\Theta(1)$ there becomes $\tau=\Theta(1/n)$ here.
By the Moore-Aronszajn Theorem~\cite{aronszajn1950theory}, there is a unique Hilbert space $\Hs$ of functions on $\Xs$ with the reproducing property: $\langle g,k(x,\cdot)\rangle_{\Hs} = g(x)$ for all $g\in \Hs$ and $x\in \Xs$.
In this paper, we do not assume $f\in \Hs$. Instead, throughout the paper, we let $\fwithin$ denote any function in $\Hs$ (e.g., a near-best sup-norm approximation of $f$), and aim to find an optimization procedure for $f$ whose performance depends on $\Hs$, $\|\fwithin\|_{\Hs}$ and, crucially,
$$\boxed{\varepsilon\coloneqq\sup_{x\in \Xs} |f(x)-\fwithin(x)|}.$$
In the rest of this work, regret is measured relative to a finite competitor set $\Xs^\circ \subset \Xs$, and we let $x_*\in \argmax_{x\in \Xs^\circ} f(x)$ (note that the algorithms may still query arbitrary points in $\Xs$).
We use a finite benchmark to avoid imposing compactness or continuity assumptions only to guarantee an optimizer, and to keep the focus on misspecification amplification rather than routine discretization arguments.
For compact continuous domains, one can take $\Xs^\circ$ to be a sufficiently fine net and add the usual approximation term between the true maximizer and the best point in the net.
\begin{definition}[Simple regret]
    For an algorithm $\mathscr A$, the simple regret under a budget
    of $n$ tries is defined as 
    $$r_n^{\mathscr A}=f(x_*)-f(x_n).$$
\end{definition}
This notion of regret applies to the first scenario (offline optimization). For the second interaction model, we use the following notion against the same competitor set.
\begin{definition}[Online regret]\label{def:online}
    For an algorithm $\mathscr A$, the online regret over $n$ interaction steps is defined as $$R_n^{\mathscr A}=\sum_{t=1}^nf(x_*)-f(x_t).$$
\end{definition}
When $\Xs^\circ$ is used as a fine discretization, it may be much larger than $n$; accordingly, our offline guarantees only scale logarithmically with $|\Xs^\circ|$ when a union bound over competitors is needed. The online confidence bounds used later are uniform over the queried regions, so the online regret bound does not pay such a factor explicitly.
In the next section, we focus on the scenario of offline optimization.

\section{Offline Optimization}\label{sec:offline}

In this section, we study offline optimization under the assumption that the query points come i.i.d. from the distribution $\inpdist$. Our goal is to isolate, for KRR, the same kind of operator-theoretic mechanism that \cite{maran2025beyond} identified for random-design least squares: misspecification amplification is controlled by the $L^\infty\to L^\infty$ norm of the relevant approximation operator. In the kernel setting, the role of the approximation operator is played by the following regularized \textit{population operator}.
\begin{equation}                
    \mathsf{P}_{\tau}f(x)\coloneqq \sum_{i=1}^\infty \frac{\mu_i f_i\phi_i(x)}{\tau+\mu_i}\qquad f_i\coloneqq \int_\Xs \phi_i(z)f(z)\diff \inpdist(z).\label{eq:popul}
\end{equation}
The link between this population operator and the uniform error is captured by the following quantity, which is known in approximation theory as the \textit{Lebesgue constant}.
\begin{equation}
\Lambda(\mathsf{P}_{\tau}):=\|\mathsf{P}_{\tau}\|_{\infty\to \infty}=\sup_{\|g\|_\infty>0}\frac{\|\mathsf{P}_{\tau} g\|_\infty}{\|g\|_\infty}.\label{eq:leb}\end{equation}
This is the kernel analogue of the projection-operator norm that appears in \cite{maran2025beyond}, and it allows us to characterize the pointwise error of KRR.
\begin{restatable}{theorem}{fundleb}\label{thm:fundleb} 
    Under \cref{ass:valboundisoker}, let $(x_t,y_t)_{t=1}^n$ be offline data with $x_t$ sampled i.i.d. from $\inpdist$ and conditionally $R$-subgaussian labels satisfying $\E[y_t|x_t]=f(x_t)$. Fix $\tau>0$, $\delta\in(0,1)$, and $\fwithin\in\Hs$, and set $\varepsilon=\|f-\fwithin\|_\infty$. Then, for any fixed $x\in \Xs$, with probability at least $1-\delta$,
    \begin{align*}
        |\mathsf{P}_{\bm x,\tau}\bm y(x)-f(x)|&\le \frac{2R\sigma_{\bm x,\tau}(x)\sqrt{\log(4/\delta)}}{\sqrt {n\tau}}+ \sigma_{\bm x,\tau}(x)\|\fwithin\|_{\Hs}+\left(\Lambda(\mathsf{P}_{\tau})+1+o_n\right)\varepsilon,
    \end{align*}
    where $o_n=\frac{4\kappa^2}{3n}\log(2/\delta)+\frac{\kappa^2 (1+\sqrt{\log(2/\delta)})}{\sqrt n}$.
\end{restatable}
The three terms in the previous theorem have standard interpretations in nonparametric statistics. First, $2(n\tau)^{-\nicefrac{1}{2}}R\sigma_{\bm x,\tau}(x)\sqrt{\log(4/\delta)}$ captures stochastic uncertainty due to the label noise and scales with the subgaussian parameter and the failure probability. Second, $\sigma_{\bm x,\tau}(x)\|\fwithin\|_{\Hs}$ represents target complexity through the RKHS norm of $\fwithin$. Lastly, the misspecification term is amplified by the Lebesgue constant. In this sense, \cref{thm:fundleb} is the KRR counterpart of the random-design least-squares picture of \cite{maran2025beyond}; the rest of the section asks when kernel structure makes this amplification substantially smaller than the generic $\sqrt{d_\text{eff}}$ scale.
From this pointwise error bound, a result for the simple regret immediately follows.
\begin{corollary}\label{cor:unif_to_regret}
    The following holds for the plug-in maximization algorithm $\mathscr A$ that estimates $f$ on $\Xs^\circ$ with KRR.
    With probability at least $1-\delta$,
    $$r_n^{\mathscr A}\lesssim \left(R\sqrt{\frac{\log(2|\Xs^\circ|/\delta)}{n\tau}}+\|\fwithin\|_{\Hs}\right)\max_{x\in \Xs^\circ}\sigma_{\bm x,\tau}(x) +\Lambda(\mathsf{P}_{\tau})\varepsilon.$$
\end{corollary}
\begin{proof}
    Apply \cref{thm:fundleb} to uniformly estimate $f$ on $\Xs^\circ$ and take the estimated maximum. The $1+o_n$ term in \cref{thm:fundleb} is $o(1)$ and is absorbed into the implicit constant of the $\lesssim$ notation.
\end{proof}

In the previous result, as in the rest of the paper, we use the notation $a_n\lesssim b_n$ to indicate that there is a universal constant $c$ such that $a_n\le c b_n$ for all $n$. We write $a_n\lesssim_m b_n$ when the implicit constant is allowed to depend on the input dimension $m$.
As the previous result shows, apart from the stochastic and epistemic uncertainty terms, the error is governed by the misspecification $\varepsilon=\|f-\fwithin\|_\infty$ multiplied by the Lebesgue constant. As the posterior variance vanishes for $n\to \infty$, this term determines the large-sample regret guarantee. The remainder of this section studies this quantity and relates it to standard complexity measures.
By Mercer's theorem, there is an infinite feature map $\bm \phi$ of components $\phi_i:\Xs\to \mathbb R$ such that the functions $\phi_i$ are orthogonal on $L^2(\Xs;\inpdist)$ and
\begin{equation}
    k(x,y)=\sum_{i=1}^\infty \mu_i \phi_i(x)\phi_i(y)\label{eq:mercer}.
\end{equation}
The sequence $(\mu_i)$ gives the eigenvalues of the integral operator induced by $k$; we call them the kernel eigenvalues. They allow us to define the effective dimension, which, for any $\tau\in (0,1]$, takes the following form
\begin{equation}
    d_\text{eff}(k|\tau)\coloneqq \sum_{i=1}^\infty\frac{\mu_i}{\tau + \mu_i}\label{eq:deff}.
\end{equation}
We use the following standard spectral size measure.
\begin{definition}
    [$\defeff$: Polynomial effective dimension] A kernel is said to have polynomial effective dimension of degree $s>0$ if, for some constant $\text{C}_2>0$ and every $\tau>0$,
    $$d_\text{eff}(k|\tau)\le \text{C}_2 \tau^{-1/s}.$$
\end{definition}
Another important assumption that will often be invoked in our analysis is the following.
\begin{definition}
    [$\defpoly$: Polynomial eigendecay] A kernel satisfies polynomial eigendecay of degree $s>0$ if, for some constant $\text{C}_1>0,$
    $$\mu_i\le \text{C}_1 i^{-s}.$$
\end{definition}
For the same $s$, one has $\defpoly \implies \defeff$. A key feature of $\defpoly$ that distinguishes it from $\defeff$ is that it is \textit{not} invariant under reordering of the eigenfunctions. While the effective dimension treats all eigenvalues in the same way, polynomial eigendecay imposes a bound that explicitly depends on the current index $i$. When the eigenfunctions are interchangeable, one can sort them by eigenvalue, but when they possess an intrinsic ordering, as for Fourier features, this distinction becomes crucial.

\subsection{Bounding the Lebesgue constant}\label{sec:lebe}

Previous analyses of misspecified linear/kernelized optimization rely on reducing the error amplification to the maximal leverage/posterior variance (see Proposition~4.5 in \cite{lattimore2020learning} and Lemma~2 in \cite{bogunovic2021misspecified}). Their approach gives an upper bound on the Lebesgue constant, which in the most favorable case reduces to the following.
\begin{restatable}{theorem}{boundlebe}\label{thm:boundlebe}
    If the eigenfunctions are bounded by one, then $$\Lambda(\mathsf{P}_{\tau})\le \sqrt{ d_\text{eff}(k|\tau)}.$$
\end{restatable}
The proof is given in \cref{sec:proofs1}.
Interestingly,
when the eigenfunctions are bounded by one, the whole \cref{thm:fundleb} admits a tighter form: one can show (see \cref{thm:fundlebtwo}) that for $n=\Omega(d_\text{eff}(k|\tau))$, with high probability and ignoring logarithmic terms,
\begin{equation}|\mathsf{P}_{\bm x,\tau}\bm y(x)-f(x)|\lesssim (R+\|\fwithin\|_\Hs)\sqrt{ \frac{d_\text{eff}(k|\tau)}{n}}+\Lambda(\mathsf{P}_{\tau})\varepsilon.\label{eq:refinederr}
\end{equation}
Repeating the strategy of \cref{cor:unif_to_regret}, \cref{thm:boundlebe} yields a regret bound of the form $r_n^{\mathscr A}\lesssim \sqrt{ d_\text{eff}(k|\tau)/n}+\sqrt{ d_\text{eff}(k|\tau)}\varepsilon$.
While this bound has the advantage of being very general, it is not satisfactory for many applications. In the case $\defeff$ of a kernel with polynomial effective dimension, the previous result only shows that $\Lambda(\mathsf{P}_{\tau})\lesssim \tau^{-\nicefrac{1}{2s}}$. Thus, if $\tau$ is driven to zero polynomially fast with $n$, the misspecification amplification still grows polynomially. In the rest of this section, we bound the Lebesgue constant directly rather than through the posterior variance, which gives tighter rates under mild assumptions.

\paragraph{Lebesgue constant under generalized Bochner-type assumptions}
The results in this subsection are inspired by Bochner's theorem, which ensures that the eigenfunctions of periodic stationary kernels correspond to the Fourier basis. The argument relies on spectral control of the Mercer basis rather than on periodicity itself. We isolate this condition first.
\begin{ass}[Spectral Lebesgue growth]\label{ass:spectrallebesgue}
    In the one-dimensional case, define
    $$B_i\coloneqq \sup_{x\in \Xs}\bigg\|\sum_{j=1}^i\phi_j(x)\phi_j(\cdot)\bigg\|_{L^1(\inpdist)}.$$
    We say that the kernel has logarithmic spectral Lebesgue growth if $B_i\lesssim \log(e+i)$ for all $i\ge 1$.
\end{ass}
This assumption is the compact-domain analogue of the Dirichlet-kernel estimate used below. It is weaker than requiring periodicity or an exact Fourier basis, but it is also not automatic for an arbitrary restriction of a stationary Euclidean kernel to a bounded domain.
One canonical setting where \cref{ass:spectrallebesgue} holds is the periodic case.
\begin{theorem}[Fourier diagonalization on the torus~\cite{seeger2004gaussian}]\label{thm:boch}
    Let $k(x,y)=\mathscr K(x-y)$ be a kernel satisfying \cref{ass:valboundisoker}. If $\Xs=[-1,1]^m$, $\inpdist$ is the uniform measure over the domain and the profile $\mathscr K$ is periodic,
    then the eigenfunctions $\phi_i$ in equation \eqref{eq:mercer} correspond to the Fourier basis.
\end{theorem}
In this periodic setting, the partial sums in \cref{ass:spectrallebesgue} are Dirichlet kernels, so $B_i\lesssim \log(e+i)$. Standard stationary kernels on $\mathbb R^m$ also admit Fourier representations by the Euclidean form of Bochner's theorem, but after restriction to a compact domain their Mercer eigenfunctions need not be Fourier modes. Thus ordinary squared-exponential, Mat\'ern, Wendland, or spectral-mixture kernels do not satisfy \cref{ass:spectrallebesgue} merely by stationarity.
Periodized versions of such kernels provide a direct ML-facing surrogate. Given a stationary kernel $k(x,y)=\mathscr K(x-y)$ on $\mathbb R^m$ and a period $2L$, one can define
$\mathscr K_L^{\mathrm{per}}(t)\coloneqq \sum_{\ell\in (2L\mathbb Z)^m}\mathscr K(t+\ell),$
whenever the sum is absolutely convergent. The resulting kernel is positive semidefinite on the torus, has Fourier eigenfunctions, and therefore satisfies the spectral-basis condition above. On a fixed domain such as $[-1,1]^m$, it differs from the original kernel only through the wrap-around tails
$$\sup_{x,y\in[-1,1]^m}|k_L^{\mathrm{per}}(x,y)-k(x,y)|\le \sum_{\ell\in(2L\mathbb Z)^m\setminus\{0\}}\sup_{t\in[-2,2]^m}|\mathscr K(t+\ell)|.$$
For kernels whose correlations decay with distance, this term becomes small when $L$ is large relative to the domain and the kernel length scale. In this sense, periodization buys exact Fourier diagonalization without substantially changing the kernel on the region where data are observed.


To improve readability, we first state the result of this section for $m=1$, postponing the product-kernel multivariate case to the end of the section.
Let $\{a_i\}_{i=1}^\infty$ be a bounded sequence of real numbers. We define the first finite difference operator as
$\Delta^1 a_i := a_{i+1} - a_i,$
which serves as a discrete analogue of the first derivative. The Lebesgue constant of the KRR operator is closely linked to the $\ell^1$ norm of the discrete derivative of its sequence of eigenvalues, as the following theorem shows.
\begin{restatable}{theorem}{lebker}\label{thm:lebker}
    Define, for every integer $i\ge 1$, $h_i=\frac{\mu_i}{\tau+\mu_i}$. Then, $\Lambda(\mathsf{P}_{\tau})\le h_1B_1+\sum_{i=1}^\infty \left|\Delta^1h_i\right|B_i$.
    In particular, under logarithmic spectral Lebesgue growth,
    $$\Lambda(\mathsf{P}_{\tau})\lesssim \sum_{i=1}^\infty \left|\Delta^1h_i\right|\log(e+i).$$
\end{restatable}
The proof is given in \cref{sec:proofs1}.
\Cref{thm:lebker} shows that if the sequence $h_i$ is sufficiently regular, then the Lebesgue constant is small. This only depends on the sequence of eigenvalues, since the regularization parameter $\tau$ does not depend on $i$. In particular, kernels with monotone spectra, meaning that higher-order Fourier features have smaller eigenvalues, satisfy a simple logarithmic bound.
\begin{restatable}{corollary}{loglam}\label{cor:loglam}
    Suppose that the kernel has logarithmic spectral Lebesgue growth and non-increasing eigenvalues. Then,
    $$\Lambda(\mathsf{P}_{\tau})\lesssim \log(e+\kappa /\tau).$$
\end{restatable}
The proof is given in \cref{sec:proofs1}.
For periodic Mat\'ern kernels, whose Fourier eigenvalues are proportional to $(1+i^2)^{-\nu-1/2}$ (equation 4.14 in \cite{seeger2004gaussian}), the previous corollary is enough. Nonetheless, for most classes, including Wendland and mixture kernels, the eigenvalue decay is not monotonic (see \cref{sec:related}).

While \cref{cor:loglam} does not apply to kernels with non-monotone eigenvalue sequences, it is natural to ask whether one can reduce to \cref{thm:lebker} by slightly changing the spectrum of $k$ without significantly altering the corresponding RKHS $\Hs$.
This motivates the following.

\begin{tcolorbox}[
    colback=gray!5,
    colframe=gray!80,
    boxrule=0pt,
    enhanced,
    unbreakable,
    sharp corners,
    boxrule=0pt,
    frame hidden,
    borderline west={2pt}{0pt}{gray!80}, 
    colback=gray!5
]
\textbf{Question.} Equation \eqref{eq:refinederr} and \cref{thm:lebker} imply that the KRR error depends on
$$\|f\|_{\Hs},\qquad d_\text{eff}(k|\tau),\qquad \sum_{i=1}^\infty \left|\Delta^1h_i\right|B_i.$$
Is it possible, under $\defpoly$ or $\defeff$, to modify the spectrum of $k$ (obtaining another kernel $\overline k$) so as to reduce the third quantity without increasing the other two?
\end{tcolorbox}

First, we show a positive answer under $\defpoly$.
\begin{restatable}{theorem}{expzero}\label{theorem:expzero}
    Assume that $k$ satisfies polynomial eigendecay $\defpoly$ for some $s>0$. Then, there is a kernel $\overline k$ with the same Mercer eigenfunctions and non-increasing eigenvalues that also satisfies $\defpoly$ with the same exponent $s$ and such that, for any $f\in \Hs$,
    $$\|f\|_{\overline \Hs}\le \|f\|_{\Hs}.$$
\end{restatable}
The proof, given in \cref{sec:proofs1}, relies on a simple envelope construction: set $\overline \mu_i \coloneqq \max_{j\ge i}\mu_j$ while keeping the same eigenfunctions, which makes the sequence non-increasing. Assuming $\defpoly$ is crucial in this context, as it extends easily to the new sequence of eigenvalues $\overline \mu_i$. As noted above, since $\defpoly \implies \defeff$, the previous result also entails that the new kernel $\overline k$ has polynomial effective dimension with the same exponent $s$ as the original kernel. If the original Mercer basis has logarithmic spectral Lebesgue growth, then the same is true for $\overline k$, because the basis has not changed. Therefore, \cref{theorem:expzero} answers one part of the previous question.

Is it possible to control the effective dimension of $\overline k$ by assuming only a polynomial bound on the effective dimension of $k$? The next result gives a partial negative answer: $\defeff$ alone is not enough to guarantee a strictly better asymptotic rate for the smoothness term. This obstruction applies not only to our simple smoothing technique, but to any candidate kernel $\overline k$ that preserves the Mercer basis.
\begin{restatable}{theorem}{lowerbound}\label{thm:lowerbound}
For every $s>1$, there is a Fourier-diagonal kernel $k$ satisfying $\defeff$ of degree $s$ such that
    $$\limsup_{\tau\to 0}\frac{\Lambda(\mathsf{P}_{\tau})}{\sqrt{d_\text{eff}(k|\tau)}}\gtrsim 1$$
    and there is no kernel $\overline k$ with the same Mercer eigenfunctions which satisfies all the following three conditions $(\mathsf i)$ Norm equivalence: for some constant $C>0$, for every $f\in \Hs$, $\|f\|_{\overline \Hs}\le C\|f\|_{\Hs}$ $(\mathsf{ii})$ $\sum_{i=1}^\infty \left|\Delta^1\overline h_i\right|\lesssim (1/\tau)^{1/s'}$ for some $s'>s$ $(\mathsf{iii})$ $d_\text{eff}(\overline k|\tau)$ is polynomial in $1/\tau$ (i.e. $\defeff$).
\end{restatable}
The proof is given in \cref{sec:proofs1}.
The previous result shows that there exist kernels that are intrinsically ill-conditioned with respect to the Lebesgue constant. The kernel $k$ constructed above satisfies $\Lambda(\mathsf{P}_{\tau})\asymp \sqrt{d_\text{eff}(k|\tau)}$, i.e., the worst-case behavior permitted by \cref{thm:boundlebe}, and one cannot in general find a same-basis kernel $\overline k$ that both preserves a comparable RKHS norm and polynomial effective dimension and also yields a strictly better asymptotic rate for $\sum_{i=1}^\infty \left|\Delta^1\overline h_i\right|$. In particular, $\defeff$ alone does not force logarithmic or polylogarithmic control of this quantity.
\paragraph{Multivariate product kernels.}
The preceding one-dimensional result has a multivariate analogue for Fourier-diagonal product kernels. As discussed above, Fourier diagonalization can be enforced by periodizing stationary kernels, with only a wrap-around perturbation on the observed domain when correlations decay. 
To state the result, let $k_0$ be a one-dimensional kernel on $[-1,1]$, and define the product kernel $k_\Pi$ on $[-1,1]^m$ by
\begin{equation}
k_\Pi(x,y)\coloneqq \prod_{j=1}^m k_0(x_j,y_j).\label{eq:product_kernel}
\end{equation}
\begin{restatable}{theorem}{lebesgueproduct}\label{thm:lebesgueproduct}
    Let $k_0(x,y)=\mathscr K_0(x-y)$ be a one-dimensional kernel on $[-1,1]$ satisfying  \cref{ass:valboundisoker}. Assume that the Mercer eigenfunctions of $k_0$ are the Fourier basis and that its eigenvalues are non-increasing. Then the population KRR operator associated with the product kernel $k_\Pi$ in \eqref{eq:product_kernel} satisfies
    $$\Lambda(\mathsf{P}_{\tau})\lesssim_m \log^{2m-1}(e+\kappa^m/\tau).$$
\end{restatable}
The proof, given in \cref{sec:proofs_multim}, follows from a multivariate version of \cref{thm:lebker}: the one-dimensional finite difference is replaced by a mixed finite difference over the coordinates of the Mercer multi-index.
This is still only polylogarithmic in the regularization scale. The dependence of the logarithmic exponent on $m$ is the expected fixed-dimensional cost of working with product spectra. It comes from bounding the total absolute mixed variation of the product multiplier near the hyperbolic surface $\prod_j\mu_{i_j}\simeq\tau$.

\section{Online Optimization}\label{sec:online}

In this section, we turn to the online problem. Here, the sample sequence $(x_t)_{t=1}^n$ is not fixed in advance, but is chosen sequentially in order to optimize $f$. This setting focuses on the online regret objective (see \cref{def:online}). The analysis of the previous section does not apply directly, as the exploration-exploitation dilemma requires visiting a highly skewed distribution of points, whereas the offline Lebesgue-constant analysis is tied to a fixed input distribution and its associated Mercer basis.
To remain consistent with the kernel-bandit literature cited below, in this section we switch back to the standard unnormalized regularization convention, so the regional KRRs use $K_t^A+\lambda I$ (equivalently $V_t^A=\lambda I+\sum_{s\le t}\phi(x_s)\phi(x_s)^\top$).

The problem of misspecification in this setting was extensively studied by
\cite{bogunovic2021misspecified}, who prove an upper bound for the regret of their algorithm \textsc{EC-GP-UCB} of the form
$$R_n^{\textsc{EC-GP-UCB}}\le \widetilde \Os(\gamma_n\sqrt{n}+\varepsilon \sqrt{\gamma_n}n).$$
Amplifying the misspecification with $\sqrt{\gamma_n}$ is particularly harmful for many kernels, such as the Mat\'ern family with parameters $m\gg \nu$, where $\gamma_n$ can be nearly linear in $n$. In such a case, only $\varepsilon \le n^{-1/2}$ would guarantee a non-vacuous bound. On the other hand, the result cannot be improved in full generality. In fact, kernelized bandits generalize linear bandits of dimension $d\approx \gamma_n$, where the regret may grow as $\varepsilon\sqrt d n$, as \cite{lattimore2020learning} shows.
\Cref{sec:floor} discusses a related obstruction for nested finite-dimensional linear approximations, where vanishing approximation error does not by itself rule out floors for feature-based bandit algorithms.
We therefore impose additional structural conditions on the kernel, natural ones that have appeared before in this literature. Under sup-norm misspecification, an $n\varepsilon$ contribution is in general inevitable; the main goal is to avoid additional multiplicative factors in front of this term.

\begin{definition}
    [$\defpolysub$: Polynomial eigendecay on subdomains] (Definition 1 in \cite{vakili2023kernelized}) A kernel satisfies polynomial eigendecay on subdomains of degree $(\alpha,s)>0$ for some constant $\text{C}_3>0$ if, when restricted to a hypercube of side $\rho\le 1$, its eigenvalues satisfy
    $\mu_i\le \text{C}_3 \rho^{\alpha}i^{-s}.$
\end{definition}
This definition strengthens the polynomial eigendecay condition $\defpoly$ by ensuring that the eigenvalues decrease whenever one restricts the domain of the kernel.
\begin{definition}
    [$\defbounded$: Bounded eigenfunctions on subdomains] (cf.\ Lemma~2 of \cite{vakili2023kernelized}) A kernel satisfies bounded eigenfunctions on subdomains for some constant $\text{C}_4>0$ if, when restricted to a hypercube of side $\rho\le 1$, its eigenfunctions satisfy $\sup_i \|\phi_i\|_\infty\le \text{C}_4$.
\end{definition}
Assuming bounded eigenfunctions is standard in the literature; see, for instance, \cite{salgia2023random,chatterji2019online,vakili2021information,riutort2023practical,liu2023estimation}. This assumption is satisfied for all kernels that admit the Fourier basis as eigenvectors, including most of those examined in the previous section.
Both assumptions hold for Mat\'ern kernels \cite{yang2020provably,vakili2023kernelized}. In particular, they satisfy $\defpolysub$ with $s=\frac{2\nu+m}{m}$ and $\alpha=2\nu$.

\subsection{Domain splitting approach}

In the random-design linear setting, \cite{maran2025beyond} makes the localization principle explicit through the Lebesgue constant of the projection operator. The same principle must be implemented differently in the online setting, because the offline Lebesgue-constant analysis does not extend directly to adaptive sampling. Instead, we exploit domain splitting to keep the number of samples per region under control. This makes the misspecification contribution in the regional confidence bounds depend on the local sample size, while subdomain eigendecay keeps the corresponding regional information gain small. A similar idea was already developed, for purposes unrelated to misspecification, by \cite{janz2020bandit}. Their algorithm $\pi$-\textsc{GP-UCB} relies on a dyadic splitting of each region whenever the number of data points it contains exceeds a given threshold. This creates $2^m$ new regions, on each of which KRR is fit anew. While they work only with Mat\'ern kernels, this idea remains valid under our assumptions.

\begin{algorithm}[t]
\footnotesize
\begin{algorithmic}[1]
\Procedure{$\pi$-\textsc{Misspec-GPUCB}}{$n,b,\lambda,B,\varepsilon,\delta$}
    \State Start with $\As_1\gets \{[-1,1]^m\}$
    \For{$t=1,\dots,n$}
        \State $A_t,x_t\gets \arg\max_{A\in \As_t,\,x\in A} \mathrm{UCB}_t(A,x)$ in \eqref{eq:ucb}
        \State Pull $x_t$ and receive $y_t$
        \If{$|N_t^{A_t}|\ge \rho_{A_t}^{-b}$}\label{algline:split}
            \State Split $A_t$ into $2^m$ sub-regions $\As_t^\circ$
            \State $\As_{t+1}\gets \As_{t}\cup \As_t^\circ \setminus \{A_t\}$
            \State Fit KRR estimator $\widehat \mu_t^A$ on all $A\in \As_{t+1}\setminus \As_{t}$
        \Else
            \State $\As_{t+1}\gets \As_{t}$
            \State Update KRR estimator $\widehat \mu_t^{A_t}$
        \EndIf
    \EndFor
\EndProcedure
\end{algorithmic}
\caption{$\pi$-\textsc{Misspec-GPUCB}}\label{alg:pigpucb}
\end{algorithm}

At each time step $t$, \cref{alg:pigpucb} maintains a collection of regions $\As_t$ partitioning the space $\Xs=[-1,1]^m$. It then fits a separate KRR on each region $A\in \As_t$, with mean estimate $\widehat \mu_{t-1}^A$ and posterior variance $\sigma_{t-1}^A(x)$. The algorithm then maximizes the following upper-confidence bound jointly over the region and the point in the region:
\begin{equation}
    \text{UCB}_t(A,x)=\widehat \mu_{t-1}^A(x)+\beta_t^A\sigma_{t-1}^A(x).\label{eq:ucb}
\end{equation}
The previous equation contains an exploration bonus that, for an upper bound $B$ on the function norm $\|\fwithin\|_\Hs$ and the maximal information gain $\gamma_{N_{t-1}^A}$ associated with the points available in the region before round $t$, takes the form
\begin{equation}
    \beta_t^A:=\frac{R}{\lambda^{1/2}}\sqrt{bm\log(4t/\delta)+1+\gamma_{N_{t-1}^A}}+B+\varepsilon\sqrt{\frac{N_{t-1}^A}{\lambda}}\label{eq:betachoice}.
\end{equation}
Here, $b$ is a parameter that determines how many points are needed to split a region; together with $\lambda$ it will be specified in the regret bound. The main feature of this upper bound relative to other works on kernelized bandits is the presence of the term $\varepsilon\sqrt{N_{t-1}^A/\lambda}$, which inflates the confidence width proportionally to the accumulated misspecification in region $A$. With this exploration bonus, we obtain the following regret bound.
\begin{restatable}{theorem}{regretbound}\label{thm:regretbound}
    Consider the online model of \cref{def:online} on $[-1,1]^m$. Under \cref{ass:valboundisoker}, let $\fwithin\in\Hs$ satisfy $\|\fwithin\|_\Hs\le B$ and $\|f-\fwithin\|_\infty\le\varepsilon$, and suppose the kernel satisfies $\defpolysub$ and $\defbounded$. Then running \cref{alg:pigpucb} with parameters $\lambda = 1$ and $b=\alpha$ yields, with probability at least $1-\delta$, a regret bound of order
    $$R_n^\mathscr{A}\le \widetilde \Os\left(n^\frac{2m+\alpha}{2m+2\alpha}+n\varepsilon\right).$$
\end{restatable}
The proof is given in \cref{sec:proofs2}.
Under the assumptions of the theorem, we have in general $\gamma_n=\widetilde \Theta(n^\frac{m}{m+\alpha})$. Therefore, the regret bound can equivalently be written as
$\widetilde{\mathcal O}(\sqrt{\gamma_n n}+n\varepsilon)$, which improves over the bound $\gamma_n\sqrt{n}+\sqrt{\gamma_n}n\varepsilon$ from \cite{bogunovic2021misspecified}. This is particularly relevant for the Mat\'ern kernel, where tight bounds for maximal information gain are known \cite{vakili2021information}, namely
$$\gamma_n=\widetilde \Theta(n^\frac{m}{m+2\nu})=\widetilde \Theta(n^\frac{m}{m+\alpha}).$$
In this case, our theorem recovers the optimal well-specified regret rate $\widetilde \Theta(n^\frac{m+\nu}{m+2\nu})$, up to the misspecification term.
For this kernel, the choice $b=\alpha=2\nu$ does not correspond to the one by \cite{janz2020bandit}, which takes $b=\frac{m+2\nu}{m+1}$. This reflects a different way to bound the information gain.

\section{Conclusions}

We studied misspecified kernelized optimization through a common localization perspective. In the offline setting, localization acts through the approximation operator and yields logarithmic misspecification amplification under monotone spectra and polylogarithmic amplification for multivariate product kernels, improving on the generic square-root control. The key technical ingredient is an explicit control of $\Lambda(\mathsf{P}_{\tau})$ via discrete spectral smoothness. This yields sharp bounds under monotone spectra, and a constructive monotone-envelope reduction covers the polynomial-eigendecay case.

In the online setting, localization acts through the query space. We analyzed a domain-splitting UCB strategy and derived regret of order $\widetilde{\mathcal O}(\sqrt{\gamma_n n}+n\varepsilon)$ under subdomain eigendecay assumptions. Relative to prior misspecified kernel-bandit bounds, this removes an additional $\sqrt{\gamma_n}$ penalty from the misspecification-dependent term.

Finally, the lower-bound result shows that polynomial effective dimension alone does not imply logarithmic misspecification amplification. Identifying the minimal structural assumptions sufficient for sharper online guarantees remains open.


\bibliographystyle{plainnat}
\bibliography{references}

\begin{thebibliography}{33}
\providecommand{\natexlab}[1]{#1}
\providecommand{\url}[1]{\texttt{#1}}
\expandafter\ifx\csname urlstyle\endcsname\relax
  \providecommand{\doi}[1]{doi: #1}\else
  \providecommand{\doi}{doi: \begingroup \urlstyle{rm}\Url}\fi

\bibitem[Aronszajn(1950)]{aronszajn1950theory}
Nachman Aronszajn.
\newblock Theory of reproducing kernels.
\newblock \emph{Transactions of the American mathematical society}, 68\penalty0 (3):\penalty0 337--404, 1950.

\bibitem[Bogunovic and Krause(2021)]{bogunovic2021misspecified}
Ilija Bogunovic and Andreas Krause.
\newblock Misspecified {Gaussian} process bandit optimization.
\newblock \emph{Advances in neural information processing systems}, 34:\penalty0 3004--3015, 2021.

\bibitem[Camilleri et~al.(2021)Camilleri, Jamieson, and Katz-Samuels]{camilleri2021high}
Romain Camilleri, Kevin Jamieson, and Julian Katz-Samuels.
\newblock High-dimensional experimental design and kernel bandits.
\newblock In \emph{International Conference on Machine Learning}, pages 1227--1237. PMLR, 2021.

\bibitem[Chatterji et~al.(2019)Chatterji, Pacchiano, and Bartlett]{chatterji2019online}
Niladri Chatterji, Aldo Pacchiano, and Peter Bartlett.
\newblock Online learning with kernel losses.
\newblock In \emph{International Conference on Machine Learning}, pages 971--980. PMLR, 2019.

\bibitem[Chowdhury and Gopalan(2017)]{chowdhury2017kernelized}
Sayak~Ray Chowdhury and Aditya Gopalan.
\newblock On kernelized multi-armed bandits.
\newblock In \emph{International Conference on Machine Learning}, pages 844--853. PMLR, 2017.

\bibitem[Dong and Yang(2023)]{dong2023does}
Jialin Dong and Lin Yang.
\newblock Does sparsity help in learning misspecified linear bandits?
\newblock In \emph{International Conference on Machine Learning}, pages 8317--8333. PMLR, 2023.

\bibitem[Du et~al.(2020)Du, Kakade, Wang, and Yang]{du2020good}
Simon~S Du, Sham~M Kakade, Ruosong Wang, and Lin~F Yang.
\newblock Is a good representation sufficient for sample efficient reinforcement learning?
\newblock In \emph{International Conference on Learning Representations}, 2020.

\bibitem[Ezzerg et~al.(2025)Ezzerg, Bogunovic, and Knoblauch]{ezzerg2025robust}
Abdelhamid Ezzerg, Ilija Bogunovic, and Jeremias Knoblauch.
\newblock Robust bayesian optimisation with unbounded corruptions.
\newblock \emph{arXiv preprint arXiv:2511.15315}, 2025.

\bibitem[Folland(2009)]{folland2009fourier}
Gerald~B Folland.
\newblock \emph{Fourier analysis and its applications}, volume~4.
\newblock American Mathematical Soc., 2009.

\bibitem[Guest(1958)]{guest1958spacing}
P.~G. Guest.
\newblock The spacing of observations in polynomial regression.
\newblock \emph{The Annals of Mathematical Statistics}, 29\penalty0 (1):\penalty0 294--299, 1958.

\bibitem[Hsu et~al.(2012)Hsu, Kakade, and Zhang]{hsu2012random}
Daniel Hsu, Sham~M Kakade, and Tong Zhang.
\newblock Random design analysis of ridge regression.
\newblock In \emph{Conference on learning theory}, pages 9.1--9.24. JMLR Workshop and Conference Proceedings, 2012.

\bibitem[Janz et~al.(2020)Janz, Burt, and Gonz{\'a}lez]{janz2020bandit}
David Janz, David Burt, and Javier Gonz{\'a}lez.
\newblock Bandit optimisation of functions in the {Mat{\'e}rn} kernel {RKHS}.
\newblock In \emph{International Conference on Artificial Intelligence and Statistics}, pages 2486--2495. PMLR, 2020.

\bibitem[Kiefer and Wolfowitz(1960)]{kiefer1960equivalence}
Jack Kiefer and Jacob Wolfowitz.
\newblock The equivalence of two extremum problems.
\newblock \emph{Canadian Journal of Mathematics}, 12:\penalty0 363--366, 1960.

\bibitem[Lattimore and Szepesv{\'a}ri(2020)]{lattimore2020bandit}
Tor Lattimore and Csaba Szepesv{\'a}ri.
\newblock \emph{Bandit algorithms}.
\newblock Cambridge University Press, 2020.

\bibitem[Lattimore et~al.(2020)Lattimore, Szepesv\'ari, and Weisz]{lattimore2020learning}
Tor Lattimore, Csaba Szepesv\'ari, and Gellert Weisz.
\newblock Learning with good feature representations in bandits and in rl with a generative model.
\newblock In \emph{International conference on machine learning}, pages 5662--5670. PMLR, 2020.

\bibitem[Li and Scarlett(2022)]{li2022gaussian}
Zihan Li and Jonathan Scarlett.
\newblock {Gaussian} process bandit optimization with few batches.
\newblock In \emph{International Conference on Artificial Intelligence and Statistics}, pages 92--107. PMLR, 2022.

\bibitem[Liu and Li(2023)]{liu2023estimation}
Zejian Liu and Meng Li.
\newblock On the estimation of derivatives using plug-in kernel ridge regression estimators.
\newblock \emph{Journal of Machine Learning Research}, 24\penalty0 (266):\penalty0 1--37, 2023.

\bibitem[Maran and Szepesv\'ari(2025)]{maran2025beyond}
Davide Maran and Csaba Szepesv\'ari.
\newblock Beyond least squares: Uniform approximation and the hidden cost of misspecification.
\newblock In \emph{The Thirty-ninth Annual Conference on Neural Information Processing Systems}, 2025.

\bibitem[Maran et~al.(2024)Maran, Metelli, Papini, and Restelli]{maran2024local}
Davide Maran, Alberto~Maria Metelli, Matteo Papini, and Marcello Restelli.
\newblock Local linearity: the key for no-regret reinforcement learning in continuous {MDP}s.
\newblock \emph{arXiv preprint arXiv:2410.24071}, 2024.

\bibitem[Parra and Tobar(2017)]{parra2017spectral}
Gabriel Parra and Felipe Tobar.
\newblock Spectral mixture kernels for multi-output gaussian processes.
\newblock \emph{Advances in Neural Information Processing Systems}, 30, 2017.

\bibitem[Riutort-Mayol et~al.(2023)Riutort-Mayol, B{\"u}rkner, Andersen, Solin, and Vehtari]{riutort2023practical}
Gabriel Riutort-Mayol, Paul-Christian B{\"u}rkner, Michael~R Andersen, Arno Solin, and Aki Vehtari.
\newblock Practical {Hilbert} space approximate {Bayesian} {Gaussian} processes for probabilistic programming.
\newblock \emph{Statistics and Computing}, 33\penalty0 (1):\penalty0 17, 2023.

\bibitem[Salgia et~al.(2021)Salgia, Vakili, and Zhao]{salgia2021domain}
Sudeep Salgia, Sattar Vakili, and Qing Zhao.
\newblock A domain-shrinking based bayesian optimization algorithm with order-optimal regret performance.
\newblock \emph{Advances in Neural Information Processing Systems}, 34:\penalty0 28836--28847, 2021.

\bibitem[Salgia et~al.(2023)Salgia, Vakili, and Zhao]{salgia2023random}
Sudeep Salgia, Sattar Vakili, and Qing Zhao.
\newblock Random exploration in bayesian optimization: Order-optimal regret and computational efficiency.
\newblock \emph{arXiv preprint arXiv:2310.15351}, 2023.

\bibitem[Seeger(2004)]{seeger2004gaussian}
Matthias Seeger.
\newblock {Gaussian} processes for machine learning.
\newblock \emph{International journal of neural systems}, 14\penalty0 (02):\penalty0 69--106, 2004.

\bibitem[Srinivas et~al.(2010)Srinivas, Krause, Kakade, and Seeger]{srinivas2010gaussian}
Niranjan Srinivas, Andreas Krause, Sham~M. Kakade, and Matthias Seeger.
\newblock {Gaussian} process optimization in the bandit setting: No regret and experimental design.
\newblock In \emph{International Conference on Machine Learning}, pages 1015--1022, 2010.

\bibitem[Tsybakov(2009)]{tsybakov2009introduction}
Alexandre~B. Tsybakov.
\newblock \emph{Introduction to Nonparametric Estimation}.
\newblock Springer, New York, 2009.

\bibitem[Vakili and Olkhovskaya(2023)]{vakili2023kernelized}
Sattar Vakili and Julia Olkhovskaya.
\newblock Kernelized reinforcement learning with order optimal regret bounds.
\newblock \emph{Advances in Neural Information Processing Systems}, 36:\penalty0 4225--4247, 2023.

\bibitem[Vakili et~al.(2021{\natexlab{a}})Vakili, Bouziani, Jalali, Bernacchia, and Shiu]{vakili2021optimal}
Sattar Vakili, Nacime Bouziani, Sepehr Jalali, Alberto Bernacchia, and Da-shan Shiu.
\newblock Optimal order simple regret for {Gaussian} process bandits.
\newblock \emph{Advances in Neural Information Processing Systems}, 34:\penalty0 21202--21215, 2021{\natexlab{a}}.

\bibitem[Vakili et~al.(2021{\natexlab{b}})Vakili, Khezeli, and Picheny]{vakili2021information}
Sattar Vakili, Kia Khezeli, and Victor Picheny.
\newblock On information gain and regret bounds in {Gaussian} process bandits.
\newblock In \emph{International Conference on Artificial Intelligence and Statistics}, pages 82--90. PMLR, 2021{\natexlab{b}}.

\bibitem[Valko et~al.(2013)Valko, Korda, Munos, Flaounas, and Cristianini]{valko2013finite}
Michal Valko, Nathaniel Korda, R{\'e}mi Munos, Ilias Flaounas, and Nelo Cristianini.
\newblock Finite-time analysis of kernelised contextual bandits.
\newblock \emph{arXiv preprint arXiv:1309.6869}, 2013.

\bibitem[Wendland(1995)]{wendland1995piecewise}
Holger Wendland.
\newblock Piecewise polynomial, positive definite and compactly supported radial functions of minimal degree.
\newblock \emph{Advances in computational Mathematics}, 4\penalty0 (1):\penalty0 389--396, 1995.

\bibitem[Whitehouse et~al.(2023)Whitehouse, Ramdas, and Wu]{whitehouse2023sublinear}
Justin Whitehouse, Aaditya Ramdas, and Steven~Z Wu.
\newblock On the sublinear regret of {GP-UCB}.
\newblock \emph{Advances in Neural Information Processing Systems}, 36:\penalty0 35266--35276, 2023.

\bibitem[Yang et~al.(2020)Yang, Jin, Wang, Wang, and Jordan]{yang2020provably}
Zhuoran Yang, Chi Jin, Zhaoran Wang, Mengdi Wang, and Michael Jordan.
\newblock Provably efficient reinforcement learning with kernel and neural function approximations.
\newblock \emph{Advances in Neural Information Processing Systems}, 33:\penalty0 13903--13916, 2020.

\end{thebibliography}


\appendix

\section{Related Work}\label{sec:related}

\paragraph{Misspecification amplification}
As noted in the introduction, controlling misspecification amplification is a well-studied problem in sequential decision-making under uncertainty, including bandits and reinforcement learning, for algorithms based on linear and richer function approximation models.
For linear function approximation, \cite{du2020good} established the first
$\sqrt d$ amplification lower bound, which was improved by \cite{lattimore2020learning}, who also established an upper bound of $\sqrt d$
based on optimal design. These lower bounds hold for a worst-case feature map, while \textit{allowing the learner to choose the data distribution}. Several approaches to circumvent this bound have been explored in different contexts. \cite{maran2024local} shows how to remove it in the case of a locally linear feature map; \cite{dong2023does} improves the $\sqrt d$ amplification under sparsity. Very recently, \cite{maran2025beyond} showed that, in misspecified linear regression under random design, the Lebesgue constant of the projection operator is the exact quantity governing amplification of the uniform approximation error, and that this dependence is intrinsic for empirical squared-loss minimizers. Our \cref{thm:fundleb} is the kernel KRR analogue of that perspective, while the rest of \cref{sec:offline} studies when kernel spectral structure forces the corresponding Lebesgue constant to be only logarithmic or polylogarithmic. 

In the broader context of kernel methods, the work most closely related to ours is \cite{bogunovic2021misspecified}, which proves a regret upper bound for misspecified kernelized bandits. Their result provides the worst-case bound that serves as the starting point of our analysis. A related paper that studies kernelized bandits beyond the realizability assumption is
\cite{ezzerg2025robust}, which focuses on corrupted observations — a distinct robustness setting in which only weaker guarantees can be achieved.

\paragraph{Kernel methods in sequential decision-making}

Kernelized bandits build on the GP-UCB and kernelized UCB lines of work \cite{srinivas2010gaussian,chowdhury2017kernelized,valko2013finite}, where regret is typically controlled through posterior variance or information gain. From the viewpoint of \cite{maran2025beyond}, the main challenge here is to recover a comparable localization mechanism under adaptive sampling. While several algorithms achieving order-optimal regret bounds are now known \cite{salgia2021domain,camilleri2021high,li2022gaussian}, only the domain-splitting approach of \cite{janz2020bandit} appears well-suited to addressing misspecification amplification as well. The reason is structural: our analysis needs a mechanism that localizes the data and keeps the number of samples per region under control, so that both the misspecification contribution in the confidence widths and the regional information gain can be controlled locally. The strategy of \cite{janz2020bandit} provides exactly this ingredient. By contrast, other order-optimal approaches treat the dependence induced by GP-UCB-style sampling more globally \cite{whitehouse2023sublinear}, and do not offer an equally direct route to the same localized misspecification control.

\paragraph{Stationary kernels with non-monotone spectrum} While standard analyses of kernelized methods typically assume eigenvalue decay, \cref{sec:lebe} shows that the Lebesgue constant also depends on eigenvalue \textit{smoothness}. While the common class of Mat\'ern kernels has a monotonically decreasing spectrum, this property does not hold in general.
Wendland kernels \cite{wendland1995piecewise} provide compactly supported, positive definite radial functions with strong locality and sparsity properties, and are widely used in approximation theory and large-scale settings. Their compact support induces an oscillatory spectral density, leading to non-monotone eigenvalue behavior\footnote{Page 5 of \cite{wendland1995piecewise} provides an explicit expression for the Fourier transform of the kernel in terms of an integral involving Bessel functions, whose oscillatory nature implies that the spectrum is not monotone.}. Another class with this feature is given by spectral mixture kernels \cite{parra2017spectral}, which explicitly model the spectrum as a mixture, enabling flexible, multimodal frequency representations.

\section{Proofs of Section~\ref{sec:offline}}\label{sec:proofs1}

\fundleb*
\begin{proof}
    Let $f_\varepsilon\coloneqq f-\fwithin$ and $\bm f_\varepsilon\coloneqq (f_\varepsilon(x_t))_{t=1}^n$.
    First, by linearity of the KRR operator, we write
    \begin{align*}
        |\mathsf{P}_{\bm x,\tau}\bm y(x)-f(x)| &= |\mathsf{P}_{\bm x,\tau}(\bm y-\bm f_\varepsilon)(x)+\mathsf{P}_{\bm x,\tau}\bm f_\varepsilon(x)-f(x)|\\
        &= |\mathsf{P}_{\bm x,\tau}(\bm y-\bm f_\varepsilon)(x)+\mathsf{P}_{\bm x,\tau}\bm f_\varepsilon(x)-\fwithin(x)-f_\varepsilon(x)|\\
        &\le |\mathsf{P}_{\bm x,\tau}(\bm y-\bm f_\varepsilon)(x)-\fwithin(x)|+|\mathsf{P}_{\bm x,\tau}\bm f_\varepsilon(x)-f_\varepsilon(x)|.
    \end{align*}
    In the first term, note that
    $$\mathbb E[y_t-f_\varepsilon(x_t)|x_t]=\fwithin(x_t),$$
    with the noise being $R$-subgaussian, by assumption.
    Therefore, by Theorem~1 of \cite{vakili2021optimal} (noting that they use $\lambda^2$ for the regularization, which corresponds to $n\tau$ here), applied with failure probability $\delta/2$, this term is bounded with probability at least $1-\delta/2$ as follows
    $$|\mathsf{P}_{\bm x,\tau}(\bm y-\bm f_\varepsilon)(x)-\fwithin(x)|\le \frac{2R\sigma_{\bm x,\tau}(x)\sqrt{\log(4/\delta)}}{\sqrt {n\tau}}+ \sigma_{\bm x,\tau}(x)\|\fwithin\|_{\Hs}.$$
    We now turn to the second term:
    $$|\mathsf{P}_{\bm x,\tau}\bm f_\varepsilon(x)-f_\varepsilon(x)|\le \|\mathsf{P}_{\bm x,\tau}\bm f_\varepsilon-\mathsf{P}_{\tau}f_\varepsilon\|_\infty+\|\mathsf{P}_{\tau}f_\varepsilon-f_\varepsilon\|_\infty.$$
    For the first part, one has
    \begin{align*}\|\mathsf{P}_{\bm x,\tau}\bm f_\varepsilon-\mathsf{P}_{\tau}f_\varepsilon\|_\infty
        &\le \kappa \|\mathsf{P}_{\bm x,\tau}\bm f_\varepsilon-\mathsf{P}_{\tau}f_\varepsilon\|_\Hs\\
        &\le \frac{4\kappa^2 \|f_\varepsilon\|_\infty}{3n}\log(2/\delta)+\frac{\kappa^2 \|f_\varepsilon\|_2(1+\sqrt{\log(2/\delta)})}{\sqrt n}\\
        &\le \frac{4\kappa^2 \varepsilon}{3n}\log(2/\delta)+\frac{\kappa^2 \varepsilon(1+\sqrt{\log(2/\delta)})}{\sqrt n}.
    \end{align*}
    The first inequality follows from the fact that, for any $g\in \Hs$, one has
    $$\|g\|_\infty=\sup_{x\in \Xs}|g(x)|\le \sup_{x\in \Xs}\|g\|_\Hs \|k(x,\cdot)\|_\Hs \le \sqrt{\kappa}\|g\|_\Hs\le \kappa \|g\|_\Hs,$$
    where the last inequality uses $\kappa\ge 1$.
    The second inequality follows from Lemma~15 of \cite{liu2023estimation}, applied with failure probability $\delta/2$, and the last one from the definition of $f_\varepsilon$.

    For the second term, instead,
    $$\|\mathsf{P}_{\tau}f_\varepsilon-f_\varepsilon\|_\infty\le \|\mathsf{P}_{\tau}f_\varepsilon\|_\infty+\|f_\varepsilon\|_\infty\le (1+\|\mathsf{P}_{\tau}\|_{\infty\to \infty})\varepsilon.$$
    A union bound over the two probabilistic estimates completes the proof.
\end{proof}

\subsection{Proofs of Section~\ref{sec:lebe}}

\begin{restatable}{theorem}{fundlebtwo}\label{thm:fundlebtwo}
    Under \cref{ass:valboundisoker}, let $(x_t,y_t)_{t=1}^n$ be offline data with $x_t$ sampled i.i.d. from $\inpdist$ and conditionally $R$-subgaussian labels satisfying $\E[y_t|x_t]=f(x_t)$. Fix $\tau>0$, $\delta\in(0,1)$, and $\fwithin\in\Hs$, and set $\varepsilon=\|f-\fwithin\|_\infty$. If $n\ge 64d_\text{eff}(k|\tau)\log(576d_\text{eff}(k|\tau)/\delta)$ and the eigenfunctions are bounded by one, then for any fixed $x\in \Xs$, with probability at least $1-\delta$,
    \begin{align*}
        |\mathsf{P}_{\bm x,\tau}\bm y(x)-f(x)|&\le  \left(2R\sqrt{\log(6/\delta)}+\sqrt{n\tau}\|\fwithin\|_\Hs\right)\sqrt{\frac{ 2d_\text{eff}(k|\tau)}{n}}+\left(\Lambda(\mathsf{P}_{\tau})+1+o_n\right)\varepsilon,
    \end{align*}
    where $o_n=\frac{4\kappa^2}{3n}\log(3/\delta)+\frac{\kappa^2 (1+\sqrt{\log(3/\delta)})}{\sqrt n}$.
\end{restatable}
\begin{proof}
    Let $f_\varepsilon\coloneqq f-\fwithin$ and $\bm f_\varepsilon\coloneqq (f_\varepsilon(x_t))_{t=1}^n$.
    First, by linearity of the KRR operator, we write
    \begin{align*}
        |\mathsf{P}_{\bm x,\tau}\bm y(x)-f(x)| &= |\mathsf{P}_{\bm x,\tau}(\bm y-\bm f_\varepsilon)(x)+\mathsf{P}_{\bm x,\tau}\bm f_\varepsilon(x)-f(x)|\\
        &= |\mathsf{P}_{\bm x,\tau}(\bm y-\bm f_\varepsilon)(x)+\mathsf{P}_{\bm x,\tau}\bm f_\varepsilon(x)-\fwithin(x)-f_\varepsilon(x)|\\
        &\le \underbrace{|\mathsf{P}_{\bm x,\tau}(\bm y-\bm f_\varepsilon)(x)-\fwithin(x)|}_{\text{T1}}+\underbrace{\|\mathsf{P}_{\bm x,\tau}\bm f_\varepsilon-\mathsf{P}_{\tau}f_\varepsilon\|_\infty}_{\text{T2}}\\
        &\qquad +\underbrace{\|\mathsf{P}_{\tau}f_\varepsilon-f_\varepsilon\|_\infty}_{\text{T3}}.
    \end{align*}
    In the rest of the proof, we bound these three terms, one by one.
    \begin{itemize}
    \item [(T1)] By Theorem~1 of \cite{vakili2021optimal}, applied with failure probability $\delta/3$,
    $$T1\le \frac{2R\sigma_{\bm x,\tau}(x)\sqrt{\log(6/\delta)}}{\sqrt {n\tau}}+\sigma_{\bm x,\tau}(x)\|\fwithin\|_{\Hs}.$$
    We now upper bound the posterior standard deviation $\sigma_{\bm x,\tau}(x)$. Using the feature map of eigenfunctions $\bm \phi$, we define $\overline{\bm \phi}=(\sqrt \mu_i\phi_i)$, which satisfies
    $$\mathbb E_{x\sim \inpdist}[\overline{\bm \phi}(x)\overline{\bm \phi}(x)]=\text{diag}(\mu_i)=:\Sigma.$$
    By Woodbury identity, one has
    \begin{align*}
        \sigma_{\bm x,\tau}(x)^2&=\overline{\bm \phi}(x)^\top\left(\frac{1}{\tau}\widehat \Sigma+I\right)^{-1}\overline{\bm \phi}(x)\qquad \widehat\Sigma=\frac{1}{n}\sum_{t=1}^n \overline{\bm \phi}(x_t)\overline{\bm \phi}(x_t)^\top.
    \end{align*}
    Since the eigenfunctions are bounded by one\footnote{which entails a leverage bound with the same quantity, in their notation} and the samples are i.i.d. from $\inpdist$, Part 1 of Theorem~16 of \cite{hsu2012random}, applied with parameter $\delta/12$, ensures that, with probability at least $1-\delta/3$,
    \begin{align*}
        &\left\|(\Sigma+\tau I)^{1/2}(\widehat \Sigma+\tau I)^{-1}(\Sigma+\tau I)^{1/2}\right\|_{\text{op}}\\
        &\qquad\le 1+\sqrt{\frac{4d_\text{eff}(k|\tau)\log(576d_\text{eff}(k|\tau)/\delta)}{n}}+\frac{d_\text{eff}(k|\tau)\log(576d_\text{eff}(k|\tau)/\delta)}{n}.
    \end{align*}
    Since, by assumption, $n\ge 64d_\text{eff}(k|\tau)\log(576d_\text{eff}(k|\tau)/\delta)$, the latter implies that the norm of the matrix $M_n\coloneqq(\Sigma+\tau I)^{1/2}(\widehat \Sigma+\tau I)^{-1}(\Sigma+\tau I)^{1/2}$ does not exceed $2$.

    This means that
    \begin{align*}
        \sigma_{\bm x,\tau}(x)^2 &=\overline{\bm \phi}(x)^\top\left(\frac{1}{\tau}\widehat \Sigma+I\right)^{-1}\overline{\bm \phi}(x)\\
        &=\tau\overline{\bm \phi}(x)^\top\left(\widehat \Sigma+\tau I\right)^{-1}\overline{\bm \phi}(x)\\
        &=\tau \overline{\bm \phi}(x)^\top(\Sigma+\tau I)^{-1/2}M_n(\Sigma+\tau I)^{-1/2}\overline{\bm \phi}(x)\\
        &\le 2\tau \overline{\bm \phi}(x)^\top\left(\Sigma+\tau I\right)^{-1}\overline{\bm \phi}(x)\\
        &=2\tau \sum_{i=1}^\infty \frac{\mu_i\phi_i(x)^2}{\tau+\mu_i}\\
        &=2\tau\sum_{i=1}^\infty \frac{\mu_i}{\tau+\mu_i}=2\tau d_\text{eff}(k|\tau).
    \end{align*}
    Therefore,
    $$T1\le \left(2R\sqrt{\log(6/\delta)}+\sqrt{n\tau}\|\fwithin\|_\Hs\right)\sqrt{\frac{2d_\text{eff}(k|\tau)}{n}}.$$
    
    \item [(T2)] Applying \cite{liu2023estimation}, as in the proof of \cref{thm:fundleb}, with failure probability $\delta/3$, gives
    \begin{align*}\|\mathsf{P}_{\bm x,\tau}\bm f_\varepsilon-\mathsf{P}_{\tau}f_\varepsilon\|_\infty
        &\le \kappa \|\mathsf{P}_{\bm x,\tau}\bm f_\varepsilon-\mathsf{P}_{\tau}f_\varepsilon\|_\Hs\\
        &\le \frac{4\kappa^2 \|f_\varepsilon\|_\infty}{3n}\log(3/\delta)+\frac{\kappa^2 \|f_\varepsilon\|_2(1+\sqrt{\log(3/\delta)})}{\sqrt n}\\
        &\le \left(\frac{4\kappa^2}{3n}\log(3/\delta)+\frac{\kappa^2 (1+\sqrt{\log(3/\delta)})}{\sqrt n}\right)\varepsilon.
    \end{align*}
        \item [(T3)] This term is bounded by $(1+\Lambda(\mathsf{P}_{\tau}))\varepsilon.$
    \end{itemize}
    A union bound over the three probabilistic estimates completes the proof.
\end{proof}

\boundlebe*
\begin{proof}
    For a function $f$, let
    $$f_i\coloneqq\int_\Xs f(x)\phi_i(x)\,\diff \inpdist(x).$$
    The Lebesgue constant corresponds to
    \begin{align*}
        \sup_{\|f\|_{\infty}=1, x\in \Xs}|\mathsf{P}_{\tau}f(x)|&=\sup_{\|f\|_{\infty}=1, x\in \Xs}\left|\sum_{i=1}^\infty\frac{\mu_if_i\phi_i(x)}{\tau+\mu_i}\right|\\
        &\le \sup_{\|f\|_{L^\infty}=1, x\in \Xs}\sqrt{\sum_{i=1}^\infty\frac{\mu_i^2\phi_i(x)^2}{(\tau+\mu_i)^2}}\sqrt{\sum_{i=1}^\infty f_i^2}\\
        &\le  \sqrt{\sum_{i=1}^\infty\frac{\mu_i^2}{(\tau+\mu_i)^2}}\\
        &\le  \sqrt{\sum_{i=1}^\infty\frac{\mu_i}{(\tau+\mu_i)}}=\sqrt{d_\text{eff}(k|\tau)},
    \end{align*}
    where the first inequality follows from Cauchy-Schwarz together with $|\phi_i(x)|\le 1$, and the second uses $\sum_{i=1}^\infty f_i^2=\|f\|_{L^2}^2\le \|f\|_{L^\infty}^2=1$.
\end{proof}

\begin{lemma}\label{lem:lebegen}
    Let $k$ be a kernel with Mercer decomposition \eqref{eq:mercer}. Then,
    $$\Lambda(\mathsf{P}_{\tau})=\sup_{x\in \Xs}\left\|\sum_{i=1}^\infty \frac{\mu_i\phi_i(x)}{\tau+\mu_i}\phi_i(\cdot)\right \|_{L^1}.$$
\end{lemma}
\begin{proof}
    By definition,
    \begin{align*}
        \Lambda(\mathsf{P}_{\tau}) &= \sup_{x\in \Xs}\sup_{\|f\|_\infty=1}\left|\sum_{i=1}^\infty \frac{\mu_i\phi_i(x)}{\tau+\mu_i}\int_\Xs \phi_i(y)f(y)\,\diff \inpdist(y)\right|\\
        &= \sup_{x\in \Xs}\sup_{\|f\|_\infty=1}\left|\int_\Xs \sum_{i=1}^\infty \frac{\mu_i\phi_i(x)}{\tau+\mu_i} \phi_i(y)f(y)\,\diff \inpdist(y)\right|\\
        &=\sup_{x\in \Xs}\left\|\sum_{i=1}^\infty \frac{\mu_i\phi_i(x)}{\tau+\mu_i}\phi_i(\cdot)\right \|_{1},
    \end{align*}
    where the last equality uses the duality between $L^1$ and $L^\infty$.
\end{proof}

\lebker*
\begin{proof}
    By \cref{lem:lebegen}, the Lebesgue constant of the population KRR can be written as follows
    $$\Lambda(\mathsf{P}_{\tau})=\sup_{x\in \Xs}\left\|\sum_{i=1}^\infty h_i\phi_i(x)\phi_i(\cdot)\right \|_{L^1}.$$
    For each fixed $x$, let
    $$S_i^x(\cdot)\coloneqq \sum_{j=1}^i\phi_j(x)\phi_j(\cdot).$$
    Abel summation gives
    $$\sum_{i=1}^\infty h_i\phi_i(x)\phi_i(\cdot)=-\sum_{i=1}^\infty \Delta^1h_iS_i^x(\cdot),$$
    where the boundary term vanishes because $h_i\to 0$.
    Therefore, taking $L^1$ norms gives
    $$\Lambda(\mathsf{P}_{\tau})\le \sum_{i=1}^\infty \left|\Delta^1h_i\right|B_i\le h_1B_1+\sum_{i=1}^\infty \left|\Delta^1h_i\right|B_i,$$
    which completes the proof.
\end{proof}

\begin{lemma}\label{lem:trace_eigen}
    Let $k$ be a kernel satisfying \cref{ass:valboundisoker} such that its eigenvalues are non-increasing. Then it satisfies polynomial eigendecay $\defpoly$ with $C_1=\kappa$ and $s=1$.
\end{lemma}
\begin{proof}
    By monotonicity, for any fixed $i\in \mathbb N$,
    $$\mu_i\le i^{-1}\sum_{j=1}^i\mu_j\le i^{-1}\sum_{j=1}^\infty\mu_j=\int_\Xs k(x,x)d\inpdist(x)\le \frac{\kappa}{i},$$
    where the trace equality uses \cref{ass:valboundisoker}.
\end{proof}

\loglam*
\begin{proof}
    Since the kernel has logarithmic spectral Lebesgue growth, there is a constant $C_B$ such that $B_i\le C_B\log(e+i)$ for all $i$. By \cref{lem:trace_eigen}, the non-increasing spectrum satisfies $\mu_i\le \kappa/i$. More generally, if $\mu_i\le C_1 i^{-s}$, \cref{thm:lebker} and monotonicity give
    \begin{align*}
        \Lambda(\mathsf{P}_{\tau})&\le C_B\sum_{i=1}^\infty \left|\Delta^1h_i\right|\log(e+i)+C_B\\
        &\lesssim C_B\left(1+\sum_{i=1}^\infty \frac{h_i}{i}\right).
    \end{align*}
    Now, take
    $i^*:=\lfloor (C_1/\tau)^{1/s}\rfloor+2$. The previous sum splits into two terms. First, as $h_i\le 1$,
    $$\sum_{i=1}^{i^*} \frac{h_i}{i}\le \log(i^*)\lesssim_s \log(e+C_1/\tau).$$
    Let us focus on the second one. For $i\ge i^*$,
    \begin{align*}
        \sum_{i=i^*}^{\infty} \frac{h_i}{i}&\le \frac{1}{\tau}\sum_{i=i^*}^{\infty} \frac{\mu_i}{i}
        \le\frac{C_1}{\tau}\sum_{i=i^*}^{\infty} i^{-s-1}
        \le\frac{C_1}{s\tau(i^*-1)^s}
        \le \frac{1}{s}.
    \end{align*}
    The claim follows with $C_1=\kappa$ and $s=1$.
    This completes the proof.
\end{proof}

\expzero*
\begin{proof}
    Let $\overline k$ be the kernel with the same Mercer eigenfunctions as $k$ and spectrum given as follows, $\forall i\in \mathbb N$, $\overline \mu_i \coloneqq \max_{j\ge i}\mu_j$. Then,
    \begin{itemize}
        \item For any $ f\in \Hs$,
        $$\|f\|_{\Hs}^2=\sum_{i=1}^\infty \frac{f_i^2}{\mu_i}\ge \sum_{i=1}^\infty \frac{f_i^2}{\overline \mu_i}=\|f\|_{\overline \Hs}^2.$$
        \item By $\defpoly$, for any $i$,
        $$\overline \mu_i= \max_{j\ge i}\mu_j\le \max_{j\ge i}\text{C}_1 j^{-s}\le \text{C}_1 i^{-s}.$$
    \end{itemize}
\end{proof}

\lowerbound*
\begin{proof}
    The proof is divided into two well-separated parts. First, we build the kernel $k$; then we prove the desired trade-off.

    \textbf{Part 1: building the kernel}

    Work on $\Xs=[-1,1]$ with the uniform measure, and fix any ordering of the real Fourier basis. The basis is orthonormal in $L^2(\inpdist)$ and uniformly bounded in $L^\infty$.
    Let $\{U_i\}_{i=1}^\infty$ be an infinite sequence of i.i.d. random variables with distribution $\text{Bern}(1/2)$. Let $k$ be a random kernel with eigenvalues
    \begin{equation}\mu_i=\begin{cases}
        \log_2(i)^{-s}U_{i}\qquad & \log_2(i)\in \mathbb N\\
        \log_2(i+1)^{-s}U_{i}\qquad & \log_2(i+1)\in \mathbb N\\
        0& \text{otherwise}
    \end{cases}.\label{eq:uvars}\end{equation}
    Since $s>1$, the eigenvalue sequence is summable almost surely; hence the preceding display defines a bounded positive semidefinite Fourier-diagonal kernel.
    By \cref{lem:lebegen}, if
    $$G_\tau(x,y)\coloneqq \sum_{i=1}^\infty \frac{\mu_i}{\tau+\mu_i}\phi_i(x)\phi_i(y),$$
    then $\Lambda(\mathsf{P}_{\tau})\ge \E_{X}\E_{Y}|G_\tau(X,Y)|$, where $X,Y$ are independent and uniformly distributed on $\Xs$.
    Now, fix $x,y\in\Xs$. The value of $G_\tau(x,y)$ is a sum of independent bounded random variables, so
    \begin{align*}
        V_\tau(x,y)&\coloneqq \text{Var}\left(G_\tau(x,y)\right)\\
        &= \sum_{i=1}^\infty \text{Var}\left(\frac{\mu_i}{\tau+\mu_i}\right)\phi_i^2(x)\phi_i^2(y).
    \end{align*}
    By the fact that $G_\tau(x,y)$ is a sum of independent bounded variables, there is a constant such that $\E[|G_\tau(x,y)-\E G_\tau(x,y)|]\ge c\sqrt{V_\tau(x,y)}$. Therefore, $\E[|G_\tau(x,y)|]\gtrsim \sqrt{V_\tau(x,y)}$.
    The uniform boundedness and orthonormality of the Fourier basis imply
    $$\E_{X,Y}V_\tau(X,Y)\asymp \sum_{\ell=1}^\infty \frac{\ell^{-2s}}{(\ell^{-s}+\tau)^2},
    \qquad
    \sup_{x,y}V_\tau(x,y)\lesssim \sum_{\ell=1}^\infty \frac{\ell^{-2s}}{(\ell^{-s}+\tau)^2}.$$
    Hence $\E_{X,Y}\sqrt{V_\tau(X,Y)}\ge \E_{X,Y}V_\tau(X,Y)/\sqrt{\sup_{x,y}V_\tau(x,y)}\gtrsim \sqrt{\sum_{\ell=1}^\infty \frac{\ell^{-2s}}{(\ell^{-s}+\tau)^2}}$.
    This formula allows us to estimate the expected Lebesgue constant of this random operator:
    \begin{align*}
        \E[\Lambda(\mathsf{P}_{\tau})]&\ge \E\left[\E_{X,Y}|G_\tau(X,Y)|\right]\\
        &\gtrsim \sqrt{\sum_{\ell=1}^\infty \frac{\ell^{-2s}}{(\ell^{-s}+\tau)^2}}\\
        &\ge \sqrt{\int_{1}^\infty \frac{1}{(1+\tau x^s)^2}dx}\\
        &= \tau^{-\frac{1}{2s}}\sqrt{\int_{\tau^{1/s}}^\infty \frac{1}{(1+t^s)^2}\diff t}
        \gtrsim \tau^{-\frac{1}{2s}}.
    \end{align*}
    At the same time, \cref{thm:boundlebe} ensures that
    $$\Lambda(\mathsf{P}_{\tau})\le \sqrt{d_\text{eff}(k|\tau)}.$$

    The effective dimension of this kernel can be upper bounded as follows.
    \begin{align*}
        d_\text{eff}(k|\tau)&=\sum_{i=1}^\infty \frac{\mu_i}{\tau+\mu_i}\\
        &\le \sum_{i=1}^\infty 1(\log_2(i)\in \mathbb N)\frac{\log_2(i)^{-s}}{\log_2(i)^{-s}+\tau}\\
        &\qquad +\sum_{i=1}^\infty 1(\log_2(i+1)\in \mathbb N)\frac{\log_2(i+1)^{-s}}{\log_2(i+1)^{-s}+\tau}\\
        &\le  2\sum_{\ell=1}^\infty \frac{\ell^{-s}}{\ell^{-s}+\tau}\lesssim \tau^{-\frac{1}{s}}.
    \end{align*}
    As there are positive constants, say $c_1$ and $c_2$, such that
    $$\mathbb E[\Lambda(\mathsf{P}_{\tau})]\ge c_1\tau^{-\frac{1}{2s}}\qquad \Lambda(\mathsf{P}_{\tau})\le c_2\tau^{-\frac{1}{2s}}\ a.s.,$$
    this means that, for $\tau_n=1/n$,
    $$\Prob(E_n)\coloneqq\Prob\left(\Lambda(\mathsf{P}_{\tau_n})\ge \frac{c_1}{2}\tau_n^{-\frac{1}{2s}}\right)\ge \frac{c_1}{2c_2-c_1}=:c_3.$$
    By continuity from above,
    \begin{align*}
        \Prob\left(\bigcap_{n\in \mathbb N}\bigcup_{\ell=n}^\infty E_{\ell}\right) &= \lim_n\Prob\left(\bigcup_{\ell=n}^\infty E_{\ell}\right)\\
        &\ge \lim_n c_3 = c_3.
    \end{align*}
    This shows that $E_\text{lim}\coloneqq\bigcap_{n\in \mathbb N}\bigcup_{\ell=n}^\infty E_{\ell}$ has positive probability over the realizations of the sequence \eqref{eq:uvars}. By the strong law for stationary $1$-dependent sequences, we also know that the set of realizations $E_\text{LLN}$ such that
    \begin{equation}
        E_\text{LLN}\coloneqq\left\{\lim_{m\to \infty}\frac{\sum_{j=1}^m U_{2^j}\cdot U_{2^{j+1}}}{m}=1/4\right\},
        \label{eq:lln}
    \end{equation}
    satisfies $\Prob(E_\text{LLN})=1$. This identity will be used in the second part of the proof. The set $E^*\coloneqq E_\text{LLN}\cap E_\text{lim}$ therefore has positive probability. We choose the kernel $k$ as any element of $E^*$. By definition of $E^*$, $\Lambda(\mathsf{P}_{\tau_n})\gtrsim \tau_n^{-1/(2s)}$ for some subsequence of $\{\tau_n\}_{n\in \mathbb N}$. This proves that, for the kernel $k$ constructed above,
    $$\limsup_{\tau \to 0}\frac{\Lambda(\mathsf{P}_{\tau})}{\sqrt{d_\text{eff}(k|\tau)}}\gtrsim1.$$

    \textbf{Part 2: proving the trade-off}

    Now, assume we build another kernel $\overline k$ with associated space $\overline \Hs$,
    $$\overline k(x,y)=\sum_{i=1}^{\infty}\overline \mu_i\phi_i(x)\phi_i(y)\qquad \overline h_i=\frac{\overline \mu_i}{\tau+\overline\mu_i},$$
    such that, for some constant $C$
    
    $$\sup_{f\in \Hs}\frac{\|f\|_{\overline \Hs}}{\|f\|_{\Hs}}\le C< +\infty.$$
    By definition of norm in $\Hs$, this means that for any $i$,
    $$\overline \mu_i=\|\phi_i\|_{\overline \Hs}^{-2}\ge C^{-2} \|\phi_i\|_{\Hs}^{-2}= C^{-2} \mu_i.$$
    Now, consider all values of $j\in \mathbb N$ such that both $U_{2^j}=1$ and $U_{2^{j+1}}=1$ in the definition of $k$. By \cref{eq:lln}, the number of such pairs for $j\le J$ is asymptotic to $J/4$.
    Taking any $j$ satisfying this property, two things may happen for $\overline k$:
    \begin{itemize}
        \item There is $2^j< i^* < 2^{j+1}$ such that $\overline \mu_{i^*}\le C^{-2}\frac{j^{-s}}{2}$. We refer to this as \textbf{scenario A}. In such a case, it follows that
        \begin{align*}
            \sum_{i=2^j}^{2^{j+1}} \left|\Delta^1\overline h_i\right|&\ge \overline h_{2^j}-\overline h_{i^*}\\
            &\ge \frac{C^{-2}j^{-s}}{C^{-2}j^{-s}+\tau}-\frac{C^{-2}j^{-s}/2}{C^{-2}j^{-s}/2+\tau}\\
            &=\frac{\tau C^2 j^s}{(\tau C^2 j^s + 1)(2\tau C^2 j^s + 1)}.\end{align*}
        \item There is no $2^j< i < 2^{j+1}$ such that $\overline \mu_i\le C^{-2}\frac{j^{-s}}{2}$. We call this \textbf{scenario B}. In this case, one has
        \begin{align*}
            \sum_{i=2^j}^{2^{j+1}} \overline h_i &\ge 2^{j-1}\frac{C^{-2}j^{-s}}{C^{-2}j^{-s}+\tau}.
        \end{align*}
    \end{itemize}
    Now, choose $J=\tau^{-1/s}$. By \cref{eq:lln}, for $\tau$ sufficiently small, the number of $j\le J$ such that both $U_{2^j}=1$ and $ U_{2^{j+1}}=1$ in the definition of $k$ is $\asymp J/4$. This means that at least one of \textbf{scenario A} or \textbf{B} happens at least $J/8$ times.
    \begin{itemize}
        \item [(A)] If this scenario happens at least $J/8$ times, we have
        \begin{align*}
            \sum_{i=1}^{\infty} \left|\Delta^1\overline h_i\right|&\ge \frac{J}{8}\frac{\tau C^2 J^s}{(\tau C^2 J^s + 1)(2\tau C^2 J^s + 1)}\\
            &\ge \frac{\tau^{-1/s}}{8}\frac{C^2}{(C^2+ 1)(2 C^2 + 1)}\gtrsim \tau^{-1/s}.
        \end{align*}
        This contradicts item $2)$ in the statement of \cref{thm:lowerbound} for every $s'>s$.
        \item [(B)] Let $j^*$ be the maximum, which cannot be lower than $J/8=\tau^{-1/s}/8$.
        In this case, the effective dimension of $\overline k$ is
        \begin{align*}
            d_\text{eff}(\overline k|\tau) & =\sum_{i=1}^\infty \overline h_i\ge \sum_{i=2^{j^*}}^{2^{j^*+1}} \overline h_i\\ &\ge 2^{j^*-1}\frac{C^{-2}{j^*}^{-s}}{C^{-2}{j^*}^{-s}+\tau}
            \asymp 2^{\frac{(1/\tau)^{1/s}}{8}}.
        \end{align*}
        Since the effective dimension grows exponentially in $\tau^{-1/s}$, conclusion 3) is false.
    \end{itemize}
\end{proof}

\subsection{Proofs for multivariate product kernels}
\label{sec:proofs_multim}
In the multivariate Fourier-diagonal setting, Mercer's decomposition takes the form
\begin{equation}
    k(x,y)=\sum_{\bm i\in \mathbb N^m} \mu_{\bm i}\prod_{j=1}^m \phi_{i_j}(x_j)\phi_{i_j}(y_j),
    \qquad \bm i=(i_1,\dots,i_m),
    \label{eq:mercer2}
\end{equation}
where $\phi_{i_j}(x_j)$ is the $i_j$-th single-variate Fourier feature applied to the $j$-th coordinate. Throughout this subsection, $\mathbb N=\{1,2,\dots\}$, and sums over $\mathbb N^m$ are understood as limits of rectangular partial sums when the summand is signed, and as the usual extended sums when the summand is nonnegative.
Let $\bm e_j$ be the $j$-th coordinate vector, define $\Delta_jh_{\bm i}=h_{\bm i+\bm e_j}-h_{\bm i}$, and set $\Delta^{\bm 1}=\Delta_1\cdots\Delta_m$.
One should think of $\Delta_j$ as a discrete partial derivative in the $j$-th coordinate, and $\Delta^{\bm 1}$ as the mixed first-order difference, analogous to $\partial_1\cdots\partial_m$. The following result gives a bound on the Lebesgue constant of the population KRR operator in this setting.

\begin{theorem}[Multivariate Lebesgue bound]\label{thm:lebkermulti}
    Let $k$ be a kernel on $[-1,1]^m$ whose Mercer decomposition has the Fourier-diagonal form \eqref{eq:mercer2}. Then the Lebesgue constant of the corresponding population KRR operator is bounded as follows:
    $$\Lambda(\mathsf{P}_{\tau})\lesssim_m\sum_{\bm i\in \mathbb N^m} \left|\Delta^{\bm 1}\left(\frac{\mu_{\bm i}}{\tau+\mu_{\bm i}}\right)\right|\prod_{j=1}^m\log(e+i_j).$$
\end{theorem}

For the proof, we need one summation-by-parts lemma.
For $\bm r,\bm i\in \mathbb N^m$, we write $\bm r\ge \bm i$ when $r_j\ge i_j$ for every $j$.
For an array $c_{\bm i}$, define the discrete antiderivative in the $k$-th coordinate by
    $$(I_kc)_{\bm i}\coloneqq\sum_{\ell=1}^{i_k}c_{i_1,\dots,i_{k-1},\ell,i_{k+1},\dots,i_m}.$$
\begin{lemma}[Multidimensional Abel summation]\label{lem:multi_abel}
    Let $(a_{\bm i})_{\bm i\in\mathbb N^m}$ be a real array and $(b_{\bm i})_{\bm i\in\mathbb N^m}$ an array of $L^1$ functions. 
     Assume that $\Delta^{\bm 1}a\in\ell^1(\mathbb N^m)$ and that $a_{\bm i}\to 0$ whenever $\|\bm i\|_\infty\to\infty$.
    If 
    $$\sum_{\bm r\in\mathbb N^m}|\Delta^{\bm 1}a_{\bm r}|\,\|(I_m\cdots I_1b)_{\bm r}\|_{L^1}<\infty,$$
    then
    $$\sum_{\bm i\in\mathbb N^m}a_{\bm i}b_{\bm i}
    =(-1)^m\sum_{\bm r\in\mathbb N^m}\Delta^{\bm 1}a_{\bm r}(I_m\cdots I_1b)_{\bm r}$$
    with convergence in $L^1$. In particular, if $b_{\bm i}(x)=\prod_{j=1}^m\psi_{j,i_j}(x_j)$ on a product space and $\Psi_{j,i}=\sum_{\ell=1}^i\psi_{j,\ell}$, then
    $$(I_m\cdots I_1b)_{\bm i}(x)=\prod_{j=1}^m\Psi_{j,i_j}(x_j).$$
\end{lemma}
\begin{proof}
    First, telescoping in each coordinate gives the tail representation
    $$a_{\bm i}=(-1)^m\sum_{\bm r\ge \bm i}\Delta^{\bm 1}a_{\bm r}.$$
    Indeed, the finite rectangular tail sum equals the alternating sum of $a$ over the outer corners of the rectangle, and the boundary terms vanish as the upper limits tend to infinity by the assumed decay of $a$. Substituting this representation gives
    \begin{align*}
        \sum_{\bm i\in\mathbb N^m}a_{\bm i}b_{\bm i}
        &=(-1)^m\sum_{\bm i\in\mathbb N^m}\sum_{\bm r\ge \bm i}\Delta^{\bm 1}a_{\bm r}b_{\bm i}\\
        &=(-1)^m\sum_{\bm r\in\mathbb N^m}\Delta^{\bm 1}a_{\bm r}\sum_{\bm i\le \bm r}b_{\bm i}\\
        &=(-1)^m\sum_{\bm r\in\mathbb N^m}\Delta^{\bm 1}a_{\bm r}(I_m\cdots I_1b)_{\bm r}.
    \end{align*}
    The identity $\sum_{\bm i\le \bm r}b_{\bm i}=(I_m\cdots I_1b)_{\bm r}$ holds because applying $I_1,\dots,I_m$ successively sums $b$ over the rectangular box $1\le i_j\le r_j$ in each coordinate. The assumed $L^1$ summability of the last series justifies the exchange of sums and the convergence in $L^1$. Finally, when $b_{\bm i}$ has product form, each $I_k$ acts only on the $k$-th factor, so applying $I_1,\dots,I_m$ gives the product of the one-dimensional partial sums.
\end{proof}

\begin{proof}[Proof of \cref{thm:lebkermulti}]
    By \cref{lem:lebegen}, the Lebesgue constant of the population KRR can be written as follows
    \begin{align*}
        \Lambda(\mathsf{P}_{\tau})&=\sup_{x\in \Xs}\left\|\sum_{\bm i\in \mathbb N^m} \frac{\mu_{\bm i}}{\tau+\mu_{\bm i}}\prod_{j=1}^m \phi_{i_j}(x_j)\phi_{i_j}(\cdot)\right \|_{L^1}\le \left\|\sum_{\bm i\in \mathbb N^m} \frac{\mu_{\bm i}}{\tau+\mu_{\bm i}}\prod_{j=1}^m \phi_{i_j}(\cdot)\right \|_{L^1}.
    \end{align*}
    We take the inner function and set $h_{\bm i}=\mu_{\bm i}/(\mu_{\bm i}+\tau)$. If the right-hand side in the statement is infinite, there is nothing to prove. Otherwise, the summability condition in \cref{lem:multi_abel} is satisfied. Moreover $h_{\bm i}\to 0$ whenever any coordinate $i_j\to\infty$, since the Fourier coefficients of the kernel vanish at infinity. Letting, for any $i\in \mathbb N$,
    $$\Phi_i(x)=\sum_{\ell=1}^i \phi_\ell(x).$$
    \Cref{lem:multi_abel} gives
    $$\sum_{\bm i\in \mathbb N^m} h_{\bm i}\prod_{j=1}^m \phi_{i_j}(\cdot)=(-1)^m\sum_{\bm i\in \mathbb N^m} \Delta^{\bm 1}h_{\bm i}\prod_{j=1}^m \Phi_{i_j}(\cdot).$$
    As in the single-variate case, the term $\Phi_i(\cdot)$ corresponds to the Dirichlet kernel, whose $L^1$ norm is $\Os(\log(e+i))$ \cite{folland2009fourier}.
    Therefore, taking $L^1$ norms gives
    $$\Lambda(\mathsf{P}_{\tau})\lesssim_m\sum_{\bm i\in \mathbb N^m} \left|\Delta^{\bm 1}\left(\frac{\mu_{\bm i}}{\tau+\mu_{\bm i}}\right)\right|\prod_{j=1}^m\log(e+i_j),$$
    which completes the proof.
\end{proof}

\begin{lemma}\label{lem:rational_log_derivative}
    Let $m\ge 1$, $\tau>0$, and $D=z\frac{d}{dz}$. Then, for all $z>0$,
    $$\left|D^m\left(\frac{z}{\tau+z}\right)\right|\lesssim_m \min\{z/\tau,\tau/z\}.$$
\end{lemma}
\begin{proof}
    Put $u=z/\tau$. Since $D=u\frac{d}{du}$, it is enough to prove the corresponding bound for $(u\frac{d}{du})^m\frac{u}{1+u}$. We claim by induction on $m$ that
    $$
    g_m(u)\coloneqq
    \left(u\frac{d}{du}\right)^m\frac{u}{1+u}=\frac{uP_m(u)}{(1+u)^{m+1}},$$
    where $P_m$ is a polynomial of degree at most $m-1$ whose coefficients depend only on $m$. The case $m=1$ has $P_1=1$. If the claim holds for $m$, applying $u\frac{d}{du}$ gives the same form with
    $$P_{m+1}(u)=\left(P_m(u)+uP_m'(u)\right)(1+u)-(m+1)uP_m(u),$$
    which has degree at most $m$. Thus the claim follows. For $u\le 1$, $g_m$ is upper bounded by a constant times $u$; for $u\ge 1$, it is bounded by a constant times $u^{-1}$. 
    Taking the larger of the constants and using that for $0\le u\le 1$, $1/u\ge 1 \ge u$, while for $u\ge 1$, $u\ge 1\ge 1/u$ gives the desired result.
\end{proof}

\begin{lemma}\label{lem:product_variation}
    Let $(\lambda_i)_{i\ge 1}$ be non-increasing and satisfy $\lambda_i\le C_1i^{-s}$ for some $C_1,s>0$. For
    $$H_\tau(y_1,\dots,y_m)\coloneqq \frac{\prod_{j=1}^m y_j}{\tau+\prod_{j=1}^m y_j},$$
    one has
    $$\sum_{\bm i\in \mathbb N^m}\left|\Delta^{\bm 1}H_\tau(\lambda_{i_1},\dots,\lambda_{i_m})\right|\prod_{j=1}^m\log(e+i_j)\lesssim_m \frac{\log^{2m-1}(e+C_1^m/\tau)}{s^m}.$$
\end{lemma}
\begin{proof}
    Let $I_i=[\lambda_{i+1},\lambda_i]$. By applying the fundamental theorem of calculus in each coordinate,
    $$\left|\Delta^{\bm 1}H_\tau(\lambda_{i_1},\dots,\lambda_{i_m})\right|
    \le \int_{\prod_j I_{i_j}}\left|\frac{\partial^m H_\tau}{\partial y_1\cdots \partial y_m}(y)\right|\,dy.$$
    If $y_j\in I_{i_j}$, then $\log(e+i_j)\lesssim s^{-1}\log(e+C_1/y_j)$. Therefore
    \begin{align*}
        &\sum_{\bm i\in \mathbb N^m}\left|\Delta^{\bm 1}H_\tau(\lambda_{i_1},\dots,\lambda_{i_m})\right|\prod_{j=1}^m\log(e+i_j)\\
        &\qquad\lesssim_m \frac{1}{s^m}\int_{(0,C_1]^m}\left|\frac{\partial^m H_\tau}{\partial y_1\cdots \partial y_m}(y)\right|\prod_{j=1}^m\log(e+C_1/y_j)\,dy.
    \end{align*}
    Put $z=\prod_{j=1}^m y_j$ and $D=z\frac{d}{dz}$. Since $y_j\partial_{y_j}=D$ on functions of $z$, we have
    $$\frac{\partial^m H_\tau}{\partial y_1\cdots \partial y_m}(y)=\frac{1}{z}D^m\left(\frac{z}{\tau+z}\right).$$
    By \cref{lem:rational_log_derivative}, the rational function on the right satisfies
    $$\left|D^m\left(\frac{z}{\tau+z}\right)\right|\lesssim_m \min\{z/\tau,\tau/z\}.$$
    We now set $y_j=C_1e^{-r_j}$, $r_j\ge 0$, and $A=C_1^m/\tau$. Then $z/\tau=Ae^{-\sum_jr_j}$ and the preceding integral is bounded, up to a constant depending only on $m$, by
    $$\int_{\mathbb R_+^m}\min\left\{Ae^{-\sum_jr_j},A^{-1}e^{\sum_jr_j}\right\}\prod_{j=1}^m (1+r_j)\,dr.$$
    Using the simplex identity
    $$\int_{\{r\in\mathbb R_+^m:\sum_jr_j=t\}}\prod_{j=1}^m (1+r_j)\,d\sigma(r)\lesssim_m (1+t)^{2m-1},$$
    this is at most
    $$\int_0^\infty \min\{Ae^{-t},A^{-1}e^t\}(1+t)^{2m-1}\,dt\lesssim_m \log^{2m-1}(e+A).$$
    This proves the claim.
\end{proof}

\lebesgueproduct*
\begin{proof}
    As in \cref{cor:loglam}, \cref{lem:trace_eigen} gives polynomial eigendecay $\mu_i^\circ\le \kappa/i$ for the single-variate spectrum.
    By \cref{thm:lebkermulti}, we have
    $$\Lambda(\mathsf{P}_{\tau})\lesssim_m\sum_{\bm i\in \mathbb N^m} \left|\Delta^{\bm 1}\left(\frac{\mu_{\bm i}}{\tau+\mu_{\bm i}}\right)\right|\prod_{j=1}^m\log(e+i_j).$$
    As the kernel factorizes $k_\Pi$, letting $\mu_i^\circ$ be the eigenvalues of $k_0$, one has $\mu_{\bm i}=\prod_{j=1}^m \mu_{i_j}^\circ$. The desired bound then follows by applying \cref{lem:product_variation} with $\lambda_i=\mu_i^\circ$, $C_1=\kappa$, and $s=1$.
    This completes the proof.
\end{proof}

\section{Proofs of Section~\ref{sec:online}}\label{sec:proofs2}

We collect the additional notation used in this section in \cref{tab:notabandit}.

\begin{table}[H]
\centering
\label{tab:notabandit}
\small
\renewcommand{\arraystretch}{1.08}
\setlength{\tabcolsep}{5pt}
\begin{tabularx}{\textwidth}{@{}p{1.3in}X@{}}
\toprule
\textbf{Symbol} & \textbf{Meaning} \\
\midrule
$A$ & Region of the input space $\Xs=[-1,1]^m$. \\
$\As_t$ & Active regions at step $t$\\
$\widetilde \As_n$ & $=\bigcup_{s=1}^n\As_s$\\
$\halt(A)$ & Maximum number of queries to points in region $A$ during the process.\\
$N_t^A$ & Number of queries to points in region $A$ after $t$ rounds.\\
$\bm x_t^A$ & Queries in region $A$ after $t$ rounds\\
$\bm y_t^A$ & Responses in region $A$ after $t$ rounds\\
$\widehat \mu_{t-1}^A$ & KRR estimator trained on $(\bm x_{t-1}^A,\bm y_{t-1}^A)$\\
$\beta_t^A$ & Exploration bonus on region $A$ at step $t$, see \eqref{eq:betachoice}.\\
$\sigma_{t-1}^{A}$ & Posterior variance on region $A$ at step $t$.\\
\bottomrule
\end{tabularx}
\end{table}

\subsection{Preliminary results}

\begin{lemma}\label{lem:miss_err}
    Given $\delta \in (0,1)$, for all $t \le n$, for all $A \in \As_t$, for all $x \in A$, we have
    \[
    |\widehat \mu_{t-1}^A(x) - \fwithin(x)| \le \beta_t^A \sigma_{t-1}^A(x),
    \]
    with probability $1 - \delta$.
\end{lemma}
\begin{proof}
    Let $\mu_{t-1}^A$ be the posterior mean obtained from the modified dataset
    $(\bm x_{t-1}^A, \fwithin(\bm x_{t-1}^A)+\bm y_{t-1}^A-f(\bm x_{t-1}^A))$.
    Then, from Lemma~2 of \cite{bogunovic2021misspecified}, 
    $$|\widehat \mu_{t-1}^A(x)-\mu_{t-1}^A(x)|\le \varepsilon\sqrt{\frac{N_{t-1}^A}{\lambda}}\sigma_{t-1}^A(x).$$
    Since the modified labels have conditional mean $\fwithin$ and preserve the same subgaussian noise level, Lemma~5 of \cite{janz2020bandit} ensures
    \[
    |\mu_{t-1}^A(x) - \fwithin(x)| \le \left(\frac{R}{\lambda^{1/2}}\sqrt{(bm)\log(4t/\delta)+1+\gamma_{N_{t-1}^A}}+B\right) \sigma_{t-1}^A(x).
    \]
    The proof is completed by triangle inequality and the definition of $\beta_t^A$.
\end{proof}

\begin{lemma}\label{lem:totalreg}
    The total number of regions does not exceed $|\widetilde \As_n|\le 9^m n^\frac{m}{m+b}$.
\end{lemma}
\begin{proof}

    Let us call $\Bs_n$ the regions that have split through the process. By definition,
    $$|\widetilde \As_n|\le (2^m+1)|\Bs_n|,$$
    where $+1$ takes into account the initial region. We can therefore focus on bounding $\Bs_n$. Let us denote $K_\ell$ the number of regions in $\Bs_n$ of side $2^{-\ell}$. By the geometry of the space,
    $$K_\ell \le 2^{m\ell}.$$
    Moreover, by the splitting rule, $2^{b\ell}$ points are necessary to split a region of side $2^{-\ell}$. Therefore, the cardinality of $\Bs_n$ can be written as
    $$|\Bs_n|=\sum_{\ell=1}^\infty K_\ell,\text{  such that}\quad K_\ell\le 2^{m\ell},\quad \sum_{\ell=1}^\infty 2^{b\ell}K_\ell\le n.$$

    The former is an optimization problem that can be solved explicitly. In fact, under the constraint $\sum_{\ell=1}^\infty 2^{b\ell}K_\ell\le n$, where $2^{b\ell}$ is strictly increasing in $\ell$, the allocation that maximizes $\sum_{\ell=1}^\infty K_\ell$ sets $K_\ell= 2^{m\ell}$ until the budget $n$ is reached, for some $L$.
    Let $L=\lceil \log_2(n)/(b+m)\rceil$. Then,
    $$\sum_{\ell=1}^L 2^{(b+m)\ell}\ge 2^{(b+m)L}\ge n.$$
    This proves that in the optimal allocation $K_\ell>0$ only for, at most, $\ell\le L$. Thus, the number of split regions does not exceed
    \begin{align*}
        |\Bs_n|&=\sum_{\ell=1}^L 2^{m\ell}
        \le 2^{m(L+1)}\\
        &=2^{m(\lceil \log_2(n)/(b+m)\rceil+1)}\le 2^{2m+m\lfloor \log_2(n)/(b+m)\rfloor}\\
        &=4^m n^\frac{m}{m+b}.
    \end{align*}
    Therefore,
    $$|\widetilde \As_n|\le (2^m+1)|\Bs_n|\le 9^m n^\frac{m}{m+b}.$$
\end{proof}

\subsection{Regret bound}

\regretbound*
\begin{proof}
    Let $A_t^\star$ be the unique active region in $\As_t$ containing $x^\star$. By \cref{lem:miss_err} and the definition of $(A_t,x_t)$,
    \begin{align*}
        f(x^\star)-f(x_t)
        &\le \fwithin(x^\star)-\fwithin(x_t)+2\varepsilon\\
        &\le \widehat \mu_{t-1}^{A_t^\star}(x^\star)+\beta_t^{A_t^\star}\sigma_{t-1}^{A_t^\star}(x^\star)-\widehat \mu_{t-1}^{A_t}(x_t)+\beta_t^{A_t}\sigma_{t-1}^{A_t}(x_t)+2\varepsilon\\
        &\le 2\beta_t^{A_t}\sigma_{t-1}^{A_t}(x_t)+2\varepsilon.
    \end{align*}
    Therefore,
    \begin{align*}
        R_n^\mathscr{A}&=\sum_{t=1}^n f(x^\star)-f(x_t)\\
        &\le \sum_{t=1}^n \left(2\beta_t^{A_t}\sigma_{t-1}^{A_t}(x_t)+2\varepsilon\right)\\
        &\le \sqrt{n\sum_{t=1}^n (2\beta_t^{A_t}\sigma_{t-1}^{A_t}(x_t))^2}+2n\varepsilon,
    \end{align*}
    where the last step uses Cauchy-Schwarz.
    We can rewrite this sum as a sum over the cover elements. For each region $A\in\widetilde\As_n$, define the monotone upper envelope
    $$\overline \beta_A\coloneqq \frac{R}{\lambda^{1/2}}\sqrt{bm\log(4n/\delta)+1+\gamma_{\halt(A)}^A}+B+\varepsilon\sqrt{\frac{\halt(A)}{\lambda}}.$$
    For every round in which region $A$ is used, $\beta_t^A\le \overline\beta_A$.
    \begin{align*}
        \sum_{t=1}^n (2\beta_t^{A_t}\sigma_{t-1}^{A_t}(x_t))^2
        &\le \sum_{t=1}^n \sum_{A\in \widetilde \As_n}\ind{x_t\in A}(2\overline\beta_A\sigma_{t-1}^{A}(x_t))^2\\
        &= \sum_{A\in \widetilde \As_n}\sum_{t=1}^n\ind{x_t\in A}(2\overline\beta_A\sigma_{t-1}^{A}(x_t))^2\\
        &\le 16\lambda \sum_{A\in \widetilde \As_n}\overline\beta_A^2\gamma_{\halt(A)}^{A},
    \end{align*}
    where the last inequality comes from Lemma~7 in \cite{janz2020bandit} and the definition of $\halt(A)$, which ensures that
    $$\sum_{t=1}^n\ind{x_t\in A}\sigma_{t-1}^{A}(x_t)^2\le 4\lambda \gamma_{\halt(A)}^{A}.$$
    As the kernel satisfies both $\defpolysub$ and $\defbounded$, Lemma~2 from \cite{vakili2023kernelized} applies for $\eta=0$. This ensures that the maximum information gain on the hypercube $A$ is bounded as    $$\gamma_{\halt(A)}^{A}\lesssim \log(\halt(A))^\frac{s-1}{s}\halt(A)^{1/s}\rho_A^{\alpha/s}.$$
    By line \ref{algline:split} of the algorithm, $\halt(A)\le \rho_{A}^{-b}$. Therefore, as we are running the algorithm for $b=\alpha$,
    $$\gamma_{\halt(A)}^{A}\lesssim \log(\halt(A))^\frac{s-1}{s}\rho_{A}^{-\alpha/s}\rho_A^{\alpha/s}\lesssim\log(n)^\frac{s-1}{s}.$$
    The other term in the product can also be reduced to the information gain on $A$. Indeed,
    \begin{align*}
        \overline\beta_A^2
        &\lesssim \left(\frac{R}{\lambda^{1/2}}\sqrt{bm\log(n/\delta)+1+\gamma_{\halt(A)}^A}+B\right)^2+\varepsilon^2\frac{\halt(A)}{\lambda}\\
        &\le \frac{R^2}{\lambda}\log(n/\delta)+\gamma_{\halt(A)}^A+B^2+\varepsilon^2\frac{\halt(A)}{\lambda}\\
        &\lesssim \frac{R^2}{\lambda}\log(n/\delta)+\log(n)^\frac{s-1}{s}+B^2+\varepsilon^2\frac{\halt(A)}{\lambda}.
    \end{align*}
    Combining these estimates gives

    \begin{align*}
        \sum_{t=1}^n (2\beta_t^{A_t}\sigma_{t-1}^{A_t}(x_t))^2
        &\le 16\lambda \sum_{A\in \widetilde \As_n}\overline\beta_A^2\gamma_{\halt(A)}^{A}\\
        &\lesssim \lambda \log(n)^2 \sum_{A\in \widetilde \As_n}\left(\frac{R^2}{\lambda}\log(1/\delta)+B^2+\varepsilon^2\frac{\halt(A)}{\lambda}\right)\\
        &\lesssim \lambda \log(n)^2 \left [\sum_{A\in \widetilde \As_n}\left(\frac{R^2}{\lambda}\log(1/\delta)+B^2\right)+\sum_{A\in \widetilde \As_n}\varepsilon^2\frac{\halt(A)}{\lambda}\right]\\
        &=\lambda \log(n)^2 |\widetilde \As_n|\left(\frac{R^2}{\lambda}\log(1/\delta)+B^2\right)+\varepsilon^2\log(n)^2\sum_{A\in \widetilde \As_n}\halt(A).
    \end{align*}
    Let $L=\lceil \log_2(n)/(b+m)\rceil$. By the proof of \cref{lem:totalreg}, no region deeper than $L$ can be created. Since the regions form a dyadic partition, every query $x_t$ belongs to at most one region of $\widetilde \As_n$ at each depth. Therefore,
    \begin{align*}
        \sum_{A\in \widetilde \As_n}\halt(A)
        &=\sum_{A\in \widetilde \As_n}\sum_{t=1}^n \ind{x_t\in A}\\
        &=\sum_{t=1}^n \sum_{A\in \widetilde \As_n}\ind{x_t\in A}\\
        &\le n(L+1)\lesssim n\log(n).
    \end{align*}
    Hence,
    \begin{align*}
        \sum_{t=1}^n (2\beta_t^{A_t}\sigma_{t-1}^{A_t}(x_t))^2
        &\lesssim \log(n)^2|\widetilde \As_n|\left(R^2\log(1/\delta)+\lambda B^2\right)+n\varepsilon^2\log(n)^3.
    \end{align*}
    
    We now apply \cref{lem:totalreg}, which ensures, for $b=\alpha$, $|\widetilde \As_n|\le 9^m n^\frac{m}{m+\alpha}$.  This results in the following regret bound:
    \begin{align*}
        R_n&\lesssim \sqrt{n\left(\log(n)^2|\widetilde \As_n|\left(R^2\log(1/\delta)+\lambda B^2\right)+n\varepsilon^2\log(n)^3\right)}+2n\varepsilon\\
        &\le 3^m\log(n)\,n^\frac{2m+\alpha}{2m+2\alpha}\sqrt{R^2\log(1/\delta)+\lambda B^2}+n\varepsilon\log(n)^{3/2}+2n\varepsilon.
    \end{align*}
\end{proof}

\subsection{Nested Classes and a Persistent Regret Floor}\label{sec:floor}

In statistical learning, when a simple model fails to adequately capture the true environment, the natural remedy is to increase the model's complexity. Consider a learner trying to optimize a true, unknown continuous reward function $f: \Xs \to \mathbb{R}$. By employing a sequence of increasingly expressive, nested linear function classes $\Fs_1 \subseteq \Fs_2 \subseteq \cdots$ (such as polynomial bases of increasing degree or Fourier features), universal approximation theorems guarantee that the uniform approximation error $\varepsilon_m(f) = \inf_{g \in \Fs_m} \|f - g\|_\infty$ will vanish as the algebraic dimension $m \to \infty$. Intuitively, increasing the complexity of our linear feature space should strictly improve our ability to learn the environment. Can we leverage this vanishing error in the sequential decision-making setting of stochastic bandits? In continuous or large action spaces, algorithms must rely on function approximation to generalize across arms. 
In the infinite-sample limit, it is natural to consider the behavior of the per-round expected regret.
Since the per-round regret may not converge, we consider 
the upper and lower asymptotic per-round regrets, defined respectively by the limit superior and limit inferior:
\[
\overline R_{\mathscr A}(f,\Fs_m) \coloneqq \limsup_{n\to\infty} \frac{R_n^{\mathscr A}(f,\Fs_m)}{n}, \quad \underline R_{\mathscr A}(f,\Fs_m) \coloneqq \liminf_{n\to\infty} \frac{R_n^{\mathscr A}(f,\Fs_m)}{n}.
\]
As the time horizon $n$ grows indefinitely, the statistical noise averages out, and the algorithm's performance is ultimately expected to be bottlenecked by the systematic bias $\varepsilon_m(f)$. Because this approximation error naturally disappears as we move up the nested hierarchy of linear spaces, one might reasonably hope that the asymptotic regret floor will vanish along with it.

To investigate whether this hope holds true, we turn to the workhorse of misspecified linear bandits: Misspecified Phased Elimination (e.g., Lattimore et al., 2020). Let $\mathscr A$ denote this standard algorithm. It operates in discrete epochs, maintaining an active set of plausible optimal actions. In each epoch, it queries a core set of exploratory actions—constructed via a G-optimal design over the active feature vectors—to bound the variance of its least-squares estimator. Crucially, to remain robust against the approximation error $\varepsilon_m(f)$, the algorithm artificially increases the size of its confidence intervals
by an additive factor of $\sqrt{m}\,\varepsilon_m(f)$ (here, quite generously, we allow the algorithm to know the misspecification level). 
By Remark~D.1 of \citet{lattimore2020learning}, we know that this specific algorithm achieves an upper asymptotic per-round regret bounded by the systematic misspecification:
\[
\overline R_{\mathscr A}(f,\Fs_m) \le C \sqrt{m} \,\varepsilon_m(f)
\]
where $C>0$ is a universal constant.%
\footnote{A small gap here is that the bound as stated in \citet{lattimore2020learning} assumes the algorithm also knows the horizon $n$, but this can be easily circumvented by a standard doubling trick.} 

Here, we immediately encounter a tension: the vanishing approximation error $\varepsilon_m(f)$ is amplified by $\sqrt{m}$, representing the statistical penalty for exploring a larger feature space.
The origin of the $\sqrt{m}$ factor is that the algorithm's confidence intervals must be wide enough to contain the true reward function.
Since the algorithm can measure $f$ at only finitely many points, it extrapolates to the rest of the action space through a least-squares estimator built from its linear model. In the worst case, the approximation error of the resulting estimate can be as large as $\sqrt{m}\,\varepsilon_m(f)$, explaining the additive inflation of the confidence width and, in turn, the extra $\sqrt{m}$ factor in the regret bound.

The above upper bound suggests that the asymptotic regret floor may not vanish, even as the approximation error $\varepsilon_m(f)$ goes to zero. Intuitively, 
when
\begin{align}
    \label{eq:paradox_condition}
c_f:=\liminf_{m\to\infty} \sqrt{m}\,\varepsilon_m(f) >0\,,
\end{align}
the algorithm's confidence intervals will never shrink below a positive threshold.
For polynomial features on an interval, the G-optimal designs used by phased elimination are Chebyshev-type designs \citep{kiefer1960equivalence,guest1958spacing}: they place support throughout the active interval rather than concentrating all mass near the empirically best arm.
Consequently, whenever the confidence floor keeps an interval of suboptimal arms active whose average gap is bounded away from zero, the algorithm keeps sampling such arms and incurs a non-vanishing regret floor.

One might first wonder whether the first ingredient of this obstruction, namely \cref{eq:paradox_condition}, can occur at all.
If it can, one may further ask whether it occurs for ordinary nested approximation classes, or only as a pathological possibility.
As it happens, \cref{eq:paradox_condition} holds already in a simple approximation-theoretic example: take $\Xs=[0,1]$, $f(x) = 1-\sqrt{x}$, so that the optimum action is at $x=0$, and let $\Fs_m$ be the space of polynomials of degree at most $m$.

By Jackson's Theorem for polynomial approximation, the uniform approximation error of an $\alpha$-Hölder continuous function $f$ on $[0,1]$ by a degree-$m$ polynomial scales as $\Os(1/m^\alpha)$.
Since $f(x) = 1 - \sqrt{x}$ is exactly $1/2$-Hölder continuous, the theorem yields $\varepsilon_m(f) \le C/m^{1/2}$. Furthermore, classical results by de La Vallée Poussin confirm that the best uniform polynomial approximation of $\sqrt{x}$ is bounded below by $c/\sqrt{m}$. Therefore:
\[
\varepsilon_m(f) = \Theta\left(\frac{1}{\sqrt{m}}\right)
\]
and thus \cref{eq:paradox_condition} is satisfied with $c_f = \liminf_{m\to\infty} \sqrt{m}\,\varepsilon_m(f) > 0$.

Thus, we have the following surprising conclusion: 
there exist natural reward functions $f$ and nested linear function classes $\Fs_m$ such that,
even though the approximation error $\varepsilon_m(f)$ vanishes as $m \to \infty$, for a reasonable bandit algorithm, the lower asymptotic per-round regret $\underline R_{\mathscr A}(f,\Fs_m)$ fails to vanish.

Note also that $f(x)=1-\sqrt{x}$, as a $1/2$-Hölder continuous function, is not particularly pathological.
Simple discretization-based bandit algorithms can achieve sublinear (vanishing per-round) regret on this problem.
Thus, the persistent regret floor is not an inherent property of the problem, but rather a consequence of the algorithm's reliance on linear function approximation and its associated confidence interval inflation.
This raises the question of whether algorithms that rely on linear function approximation are fundamentally limited in their ability to leverage vanishing approximation error:

\paragraph{Open Problem: The Universal Nested-Class Paradox}
Does there exist a sequence of nested linear feature maps $\phi_m : \Xs \to \mathbb{R}^m$ such that
for $m$ large enough, $\phi_m$ is injective, and such that, for any bandit algorithm $\mathscr A$ that is given access to the feature map $\phi_m$ and the misspecification error $\varepsilon_m(f)$ as input,
and which satisfies $R_n^{\mathscr A}(f,\Fs_m) = \widetilde{\Os}(m \sqrt{n}+ \sqrt{m}n \varepsilon_m(f))$ as $n\to\infty$ for all $f\in\Fs_m$,
there exists a bounded reward function $f:\Xs \to [0,1]$ satisfying:
\begin{enumerate}
    \item The uniform approximation error vanishes: $\lim_{m\to\infty} \varepsilon_m(f) = 0$.
    \item The lower asymptotic per-round regret fails to vanish: $\liminf_{m\to\infty}\underline R_{\mathscr A}(f,\Fs_m) > 0$.
\end{enumerate}
Above, $\Fs_m = \{ x \mapsto \langle w, \phi_m(x) \rangle : w \in \mathbb{R}^m \}$ and $\widetilde{\Os}$ hides logarithmic factors in $m$ and $n$.

Note that in nonparametric regression, series estimators — a form of linear function approximation — achieve minimax-optimal rates in both the $L^2$ and $L^\infty$ norms, up to logarithmic factors \citep{tsybakov2009introduction}, suggesting that any such limitation is specific to the bandit setting.


\newpage
\section*{NeurIPS Paper Checklist}

\begin{enumerate}

\item {\bf Claims}
    \item[] Question: Do the main claims made in the abstract and introduction accurately reflect the paper's contributions and scope?
    \item[] Answer: \answerYes{} 
    \item[] Justification: Proofs are provided for every claim made.

\item {\bf Limitations}
    \item[] Question: Does the paper discuss the limitations of the work performed by the authors?
    \item[] Answer: \answerYes{} 
    \item[] Justification: We clearly state the assumptions behind any result. Moreover, comparison with the state of the art is performed, enhancing in which scenarios our theorem do not apply and previous work does.

\item {\bf Theory assumptions and proofs}
    \item[] Question: For each theoretical result, does the paper provide the full set of assumptions and a complete (and correct) proof?
    \item[] Answer: \answerYes{} 
    \item[] Justification: All proofs are reported in the appendix.

    \item {\bf Experimental result reproducibility}
    \item[] Question: Does the paper fully disclose all the information needed to reproduce the main experimental results of the paper to the extent that it affects the main claims and/or conclusions of the paper (regardless of whether the code and data are provided or not)?
    \item[] Answer: \answerNA{} 
    \item[] Justification: No experiments included.

\item {\bf Open access to data and code}
    \item[] Question: Does the paper provide open access to the data and code, with sufficient instructions to faithfully reproduce the main experimental results, as described in supplemental material?
    \item[] Answer: \answerNA{} 
    \item[] Justification: No experiments.

\item {\bf Experimental setting/details}
    \item[] Question: Does the paper specify all the training and test details (e.g., data splits, hyperparameters, how they were chosen, type of optimizer) necessary to understand the results?
    \item[] Answer: \answerNA{} 
    \item[] Justification: No experiments.

\item {\bf Experiment statistical significance}
    \item[] Question: Does the paper report error bars suitably and correctly defined or other appropriate information about the statistical significance of the experiments?
    \item[] Answer: \answerNA{} 
    \item[] Justification: No experiments.

\item {\bf Experiments compute resources}
    \item[] Question: For each experiment, does the paper provide sufficient information on the computer resources (type of compute workers, memory, time of execution) needed to reproduce the experiments?
    \item[] Answer: \answerNA{} 
    \item[] Justification: No experiments.
    
\item {\bf Code of ethics}
    \item[] Question: Does the research conducted in the paper conform, in every respect, with the NeurIPS Code of Ethics \url{https://neurips.cc/public/EthicsGuidelines}?
    \item[] Answer: \answerYes{} 
    \item[] Justification: Standard theory paper, nothing to declare.

\item {\bf Broader impacts}
    \item[] Question: Does the paper discuss both potential positive societal impacts and negative societal impacts of the work performed?
    \item[] Answer: \answerNA{} 
    \item[] Justification: This paper focuses on a well-known theoretical problem in optimizing with kernelized methods. There are no particular direct applications.
    
\item {\bf Safeguards}
    \item[] Question: Does the paper describe safeguards that have been put in place for responsible release of data or models that have a high risk for misuse (e.g., pre-trained language models, image generators, or scraped datasets)?
    \item[] Answer: \answerNA{} 
    \item[] Justification: Same as before.

\item {\bf Licenses for existing assets}
    \item[] Question: Are the creators or original owners of assets (e.g., code, data, models), used in the paper, properly credited and are the license and terms of use explicitly mentioned and properly respected?
    \item[] Answer: \answerNA{} 
    \item[] Justification: No existing asset is used, related papers are quoted.

\item {\bf New assets}
    \item[] Question: Are new assets introduced in the paper well documented and is the documentation provided alongside the assets?
    \item[] Answer: \answerNA{} 
    \item[] Justification: As before.

\item {\bf Crowdsourcing and research with human subjects}
    \item[] Question: For crowdsourcing experiments and research with human subjects, does the paper include the full text of instructions given to participants and screenshots, if applicable, as well as details about compensation (if any)? 
    \item[] Answer: \answerNA{} 
    \item[] Justification: No crowdsourching.

\item {\bf Institutional review board (IRB) approvals or equivalent for research with human subjects}
    \item[] Question: Does the paper describe potential risks incurred by study participants, whether such risks were disclosed to the subjects, and whether Institutional Review Board (IRB) approvals (or an equivalent approval/review based on the requirements of your country or institution) were obtained?
    \item[] Answer: \answerNA{} 
    \item[] Justification: No human subjects.

\item {\bf Declaration of LLM usage}
    \item[] Question: Does the paper describe the usage of LLMs if it is an important, original, or non-standard component of the core methods in this research? Note that if the LLM is used only for writing, editing, or formatting purposes and does \emph{not} impact the core methodology, scientific rigor, or originality of the research, declaration is not required.
    \item[] Answer: \answerNA{} 
    \item[] Justification: 

\end{enumerate}

\end{document}